\documentclass{article}

\PassOptionsToPackage{numbers, compress, sort}{natbib}
\usepackage[main, final]{neurips_2026}
\usepackage[pagebackref=true,colorlinks,citecolor=brown]{hyperref}
\usepackage[font=small]{caption}

\usepackage[utf8]{inputenc} % allow utf-8 input
\usepackage[T1]{fontenc}    % use 8-bit T1 fonts
\usepackage{hyperref}       % hyperlinks
\usepackage{url}            % simple URL typesetting
\usepackage{booktabs}       % professional-quality tables
\usepackage{amsfonts}       % blackboard math symbols
\usepackage{nicefrac}       % compact symbols for 1/2, etc.
\usepackage{microtype}      % microtypography
\usepackage{amsmath, amssymb, amsthm}         % math
\usepackage{xcolor}         % colors
\usepackage{graphicx}         
\usepackage{comment}   
\usepackage{multirow, multicol}
\usepackage{makecell}
\usepackage{subcaption}
\usepackage{tikz}
\usepackage{fontawesome}

\renewcommand{\paragraph}[1]{\noindent\textbf{#1}}

\usepackage{algorithm}
\usepackage{algpseudocode}
\usepackage{array}

\newcommand{\lkey}[1]{\textcolor{#1}{\rule[0.5ex]{1.0em}{1.5pt}}\,}
\definecolor{c1}{HTML}{5DBDF3}
\definecolor{c2}{HTML}{FAB84E}
\definecolor{c3}{HTML}{009E73}
\definecolor{c4}{HTML}{D55E00}
\definecolor{c5}{HTML}{F0E442}
\definecolor{c6}{HTML}{76B7B2}
\definecolor{c7}{HTML}{8C564B}
\definecolor{c8}{HTML}{6A3D9A}

\usepackage[framemethod=TikZ]{mdframed}
\usepackage{amsthm}
\usepackage[dvipsnames]{xcolor}

\definecolor{propbackground}{HTML}{EDF2FE}
\definecolor{propborder}{HTML}{2A3487}

\newmdtheoremenv[
  backgroundcolor=propbackground,  % #8170B2
  linecolor=propborder,                 % black border
  linewidth=0.5pt,
  roundcorner=4pt,                 % more rounded corners
  innertopmargin=6pt,
  innerbottommargin=6pt,
  innerleftmargin=8pt,
  innerrightmargin=8pt
]{theorem}{Theorem}[section]      % env name, printed label, numbering restarts each section

\definecolor{algcomment}{gray}{0.45}

\newcommand{\codeurl}{https://github.com/haldunbalim/GTPlanning}

\title{Model-Based Online Decision Making via Generative Trajectory Planning}

\author{%
Haldun Balim$^{1}$
\quad Yilun Du$^{1 \dagger}$
\quad Na Li$^{1 \dagger}$\\
$^1$Harvard University\\
\texttt{\{hbalim@fas, ydu@seas, nali@seas\}.harvard.edu}\\[0.5em]
\href{https://haldunbalim.github.io/GTPlanning}{\faGlobe\ Website}
\quad
\href{https://github.com/haldunbalim/GTPlanning}{\faGithub\ Code}
}

\begin{document}

\maketitle

\begingroup
\renewcommand{\thefootnote}{}
\footnotetext{$\dagger$ denotes equal advising.}
\endgroup

\begin{abstract}
Online decision-making faces a tension between efficient inference and long-horizon reasoning. Value-based methods rely on short-horizon bootstrapping that limits the propagation of long-term information, while trajectory optimization methods plan explicitly but must solve a high-dimensional optimization problem at every decision step. We propose Generative Trajectory Planning (GTP), a model-based framework that addresses this tension by learning a generative model as a proposal distribution over action sequences. Instead of optimizing sequences from scratch, GTP samples candidate sequences from a diffusion prior and refines them using learned dynamics, reward, and value models, enabling efficient trajectory-level reasoning. Across continuous control and long-horizon manipulation tasks, GTP matches or exceeds state-of-the-art model-free, diffusion-policy, and model-based baselines, and remains stable in regimes where strong baselines degrade. Our results suggest that generative models are best understood not as policies, but as structured proposal distributions for planning, offering a new perspective on how to integrate generative modeling with model-based decision-making.
\end{abstract}

\vspace{-15pt}
\section{Introduction}
\vspace{-5pt}

Reinforcement learning (RL) has emerged as the dominant framework for online decision-making~\citep{sutton2018reinforcement, silver2016mastering}, with strong empirical success across domains such as robotics and continuous control~\citep{kober2013reinforcement, levine2016end, lillicrap2015continuous, haarnoja2018soft}. At its core, RL compresses future outcomes into a scalar value function and optimizes single-step policies against this objective~\citep{sutton1988learning}. While effective when reward signals are informative, this formulation exhibits structural limitations in long-horizon settings: reliance on short-horizon bootstrapping slows the propagation of long-term information and accumulates estimation errors~\citep{mnih2015human, schulman2015high}, while exploration driven by stochasticity at the level of individual actions~\citep{lillicrap2015continuous, haarnoja2018soft} limits the discovery of coherent behaviors. Model-based RL aims to mitigate these issues by incorporating learned predictive models into policy and value learning~\citep{deisenroth2011pilco, hafner2023mastering, Schrittwieser2020, janner2019trust}. In practice, however, these models are typically used indirectly, to shape value targets or generate synthetic data, rather than to enable explicit reasoning over a long-horizon. As a result, decision making remains driven by local updates, and the core limitations of short-horizon bootstrapping and action-level exploration persist.

In contrast, trajectory optimization (TO) explicitly optimizes action sequences, allowing long-horizon planning without relying on bootstrapped value functions~\citep{williams2017information, rubinstein1999cross}. However, optimizing high-dimensional, non-convex objectives at every step is computationally expensive and susceptible to local minima. Moreover, in online settings, planning relies on learned dynamics~\citep{chua2018deep, du2021modelbasedplanningenergy}, where prediction errors accumulate over long horizons and can be exploited by the optimizer~\citep{talvitie2014model, janner2019trust}. As a result, TO occupies the opposite end of the design spectrum from RL, as it enables explicit trajectory-level reasoning but is computationally intensive and degrades under model error.

A promising direction from generative modeling reframes planning as conditional trajectory generation~\citep{janner2022planning, ajay2023conditional, chi2023diffusion, balim2025model}: models learn distributions over high-return trajectories and sample candidates directly, achieving strong offline performance with large datasets~\citep{janner2022planning, ajay2023conditional}. In online settings, however, progress is limited as prior work largely treats generative models as policy replacements trained with standard RL objectives~\citep{ma2025efficient, wang2024diffusion, psenka2023learning, yang2023policy}, inheriting their shortcomings. As a result, the long-horizon and expressive benefits of generative planning remain underutilized online.

%\Lina{quickly adding a phrase to explain what online setting you were referring to or actually we should add a short phrase for ``online'' in the first paragraph.}\HB{Addressed}

\begin{figure}[t]
    \centering
    \includegraphics[width=\linewidth]{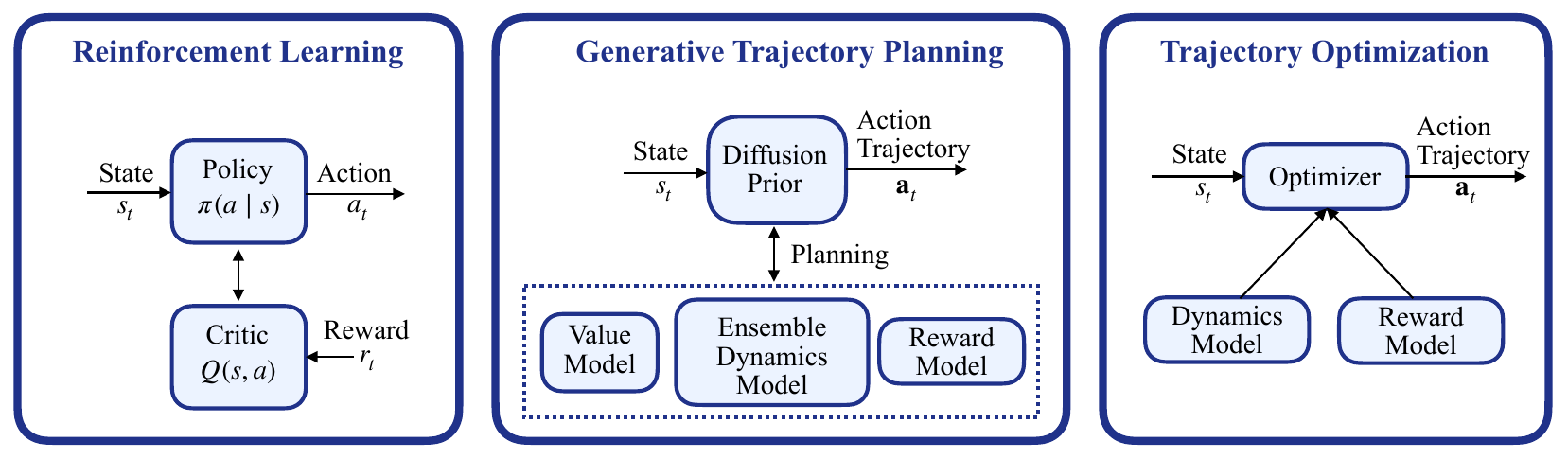}
    \vspace{-15pt}
    \caption{\textbf{Generative Trajectory Planning.} GTP combines the data-driven nature of RL with the long-horizon reasoning of trajectory optimization. A learned diffusion prior proposes candidate action sequences, which are evaluated and iteratively refined through learned models. By using a generative prior, GTP avoids computational cost of optimization while preserving long-horizon reasoning.}
    \label{fig:teaser}
    \vspace{-15pt}
\end{figure}

Taken together, these observations suggest that online decision making lacks a framework that combines efficient inference with explicit trajectory-level reasoning. We address this gap by casting decision making as generative inference, where a learned generative model defines a \textit{proposal distribution} over plausible action sequences. At each timestep, candidates are sampled from this distribution rather than optimized from scratch. These candidates are then refined using learned value and dynamics models, amortizing long-horizon reasoning into the generative model. To enable effective exploration, we characterize trajectory-level uncertainty using ensemble dynamics and act optimistically upon these estimates. The resulting framework GTP offers, a unified approach to integrate generative modeling with model-based decision making. 

In summary, our contributions are (1) we propose Generative Trajectory Planning (GTP), a model-based online decision-making framework that uses a generative model as a proposal distribution over action sequences, enabling trajectory-level reasoning without per-step optimization; (2) we develop an iterative planning procedure that interleaves beam search with diffusion-based refinement, progressively improving sampled action sequences under the learned prior and dynamics; (3) we show that GTP outperforms classical RL, diffusion-policy, and model-based RL baselines on locomotion benchmarks, improves sample efficiency on long-horizon manipulation, and scales to visual manipulation through latent representations, where standard RL methods fail to make progress.

\vspace{-10pt}
\section{Preliminaries}\label{sec:motivation}
\vspace{-5pt}
\subsection{Online Decision Making}
\vspace{-6pt}

\paragraph{Problem formulation.} We consider an infinite-horizon Markov Decision Process~\citep{mdp-book} (MDP) defined by the tuple $(\mathcal{S}, \mathcal{A}, \mathcal{P}, r, \gamma, \mu_0)$, where $\mathcal{S} \subseteq \mathbb{R}^n$ and $\mathcal{A} \subseteq \mathbb{R}^m$ are continuous state and action spaces, $\mathcal{P} : \mathcal{S} \times \mathcal{A} \to \Delta(\mathcal{S})$ is the transition function, $r : \mathcal{S} \times \mathcal{A} \to \mathbb{R}$ is the reward function, $\gamma \in [0, 1)$ is a discount factor, and $\mu_0$ is the initial state distribution. At each timestep $t$, the agent observes state $s_t$, takes an action $a_t$, receives reward $r_t = r(s_t, a_t)$, and transitions to $s_{t+1}\sim \mathcal{P}(\cdot | s_t, a_t)$. The goal is to learn a decision rule that maximizes the expected discounted return, $\mathbb{E}\left[\sum_{t=0}^{\infty} \gamma^t r_t\right]$, from data collected through \textit{online} interaction without knowing the transition model $\mathcal{P}$ and reward model $r$. 

%\vspace{-5pt}
\paragraph{Model predictive control.}
Given access to the dynamics $\mathcal{P}$ and reward $r$, MPC casts decision-making as an optimization problem solved at each timestep from the current state $s_t$:
\begin{equation}\label{eq:mpc}
\mathbf{a}_{t}^\star = \arg\max_{\mathbf{a}_{t}} \; \mathbb{E}\!\left[\sum_{h=0}^{H-1} \gamma^h r(s_{t+h}, a_{t+h}) + \gamma^H V(s_{t+H})\right],
\end{equation}
where $\mathbf{a}_t \triangleq (a_t, \ldots, a_{t+H-1})$ is an action sequence of horizon $H$, the expectation is taken under $\mathcal{P}$, and $V$ is an optional terminal value bootstrapping returns beyond the horizon. The first action $a_t^\star$ is executed and the procedure repeats. A common class of approaches for solving~\eqref{eq:mpc} uses sampling-based methods~\citep{rubinstein1999cross,williams2017information} that iteratively refine a parametric proposal toward the objective.

In the online setting, $\mathcal{P}$ and $r$ are unknown and must be replaced with learned dynamics $f_\phi$ and reward $r_\xi$; we similarly replace the terminal value by a learned $V_\psi$. Substituting these elements into~\eqref{eq:mpc} yields the planning objective the agent actually optimizes at each step. Because all three models are trained on the agent's own collected transitions, effective planning must balance exploiting current predictions against exploring regions where $f_\phi$ is uncertain; we drive this trade-off through sequence-level uncertainty over the learned dynamics in the planning objective.

Even with accurate models, the planner’s proposal remains a bottleneck: simple parametric families (e.g., Gaussians) cannot capture the multimodal, temporally correlated structure of high-return action sequences. We replace this with a learned generative model that concentrates on high-return trajectories and is reused across steps, amortizing planning. At decision time, candidates are sampled from this prior and refined using $f_\phi$, $r_\xi$, and $V_\psi$. We instantiate the prior with a diffusion model, enabling flexible modeling of high-dimensional, multimodal action distributions~\citep{janner2022planning, chi2023diffusion}.

\vspace{-7pt}
\subsection{Diffusion Models for Action Trajectory Generation}
\vspace{-8pt}
We use diffusion models to represent expressive distributions over
action sequences conditioned on the current state. Diffusion models
consist of a forward noising process $q$ that gradually corrupts
action sequences with Gaussian noise and a learned reverse process
$p_\theta$ that iteratively denoises them to generate new trajectories.
Given an action sequence $\mathbf{a}_t^0$, the forward noising kernel
at diffusion step $k$ is
$q_{k|0}(\mathbf{a}_t^k|\mathbf{a}_t^0)
=\mathcal{N}(\sqrt{\bar\alpha_k}\mathbf{a}_t^0,
(1-\bar\alpha_k)I)$,
where $\bar\alpha_k=\prod_{j=1}^k(1-\beta_j)$ is determined by a fixed schedule. Starting from Gaussian noise, the reverse process uses a learned noise predictor $\epsilon_\theta(\mathbf{a}_t^k,k,s_t)$
to generate an action sequences. The model is trained using a weighted
objective~\citep{ma2025efficient}:
\begin{equation}
    \mathcal{L}(\theta) =
    \mathbb{E}_{\mathbf{a}_t^0,k,\epsilon}
    \left[
    w(\mathbf{a}_t^0,s_t)
    \left\|\epsilon-\epsilon_\theta(\mathbf{a}_t^k,k,s_t)\right\|^2
    \right].
\end{equation}
The weights $w$ determine the target trajectory distribution. Uniform weights recover standard diffusion training, and non-uniform weights shift the target toward high-quality trajectories.
Unlike prior work that derives weights from
value functions~\citep{ma2025efficient,ding2024diffusion}, we use
trajectory-level scores (Sec.~\ref{sec:training}).

\begin{figure}[t]
    \centering
    \includegraphics[width=\linewidth]{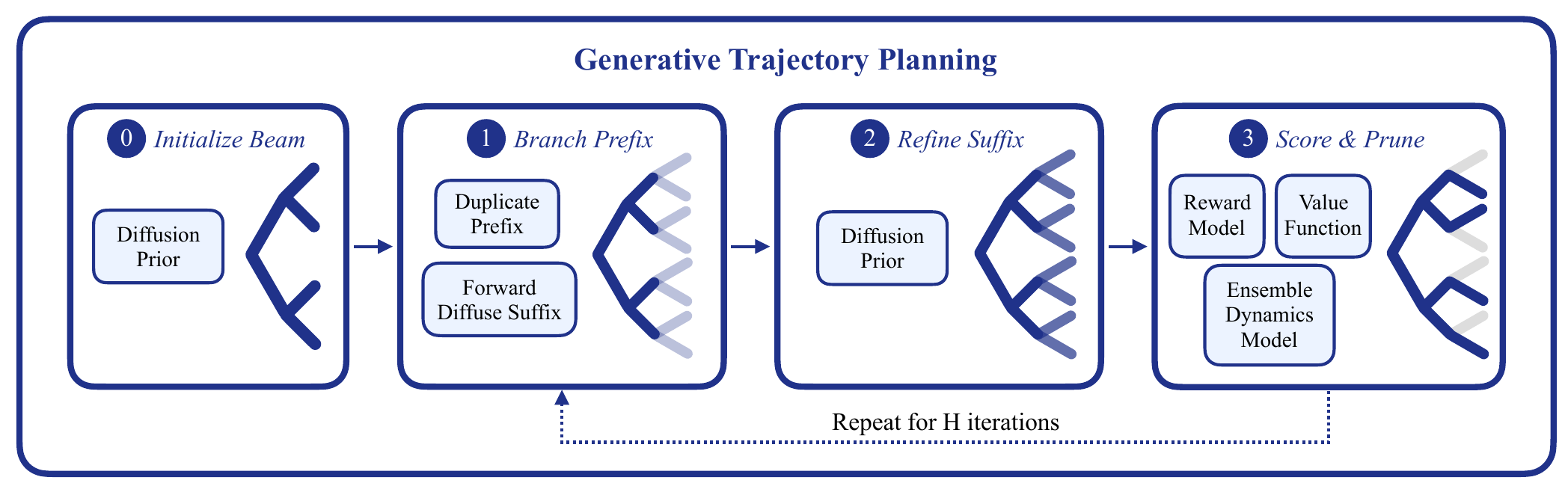}
    \vspace{-18pt}
    \caption{\textbf{Generative Trajectory Planning (GTP).}
    Planning proceeds via iterative beam search. Starting from an initial beam, GTP repeatedly \textit{(1) branches} candidate trajectory prefixes by duplicating, \textit{(2) refines} action suffixes through partial forward diffusion and denoising under the diffusion prior, and \textit{(3) prunes} candidates using rollouts through learned dynamics, reward, and value models. The top-scoring trajectories are retained and expanded at the next depth, progressively constructing long-horizon plans toward high-return regions.}
    \label{fig:framework}
\vspace{-15pt}
\end{figure}

\vspace{-12pt}
\section{Generative Trajectory Planning}
\vspace{-8pt}
Our framework casts decision making as planning over learned models. Given the current state $s_t$, candidate action sequences are sampled from a generative prior and evaluated using learned models to select high-return trajectories. Our method consists of four learned components. First, a \emph{generative action-sequence prior} $p_\theta(\mathbf{a}_t \mid s_t)$ defines a state-conditioned distribution over $H$-step action sequences and proposes candidates for planning. Second, a \emph{dynamics model} $f_\phi : \mathcal{S} \times \mathcal{A} \to \mathcal{S}$ predicts the next state, $\hat{s}_{h+1} = f_\phi(s_h, a_h)$, and is rolled out autoregressively to simulate candidate trajectories. We use an ensemble $\{f_{\phi_i}\}_{i=1}^{N}$ whose disagreement provides trajectory-level uncertainty for exploration. Third, a \emph{reward model} $r_\xi : \mathcal{S} \times \mathcal{A} \to \mathbb{R}$ predicts rewards $\hat{r}_h = r_\xi(s_h, a_h)$ along simulated trajectories. Finally, a \emph{value function} $V_\psi : \mathcal{S} \to \mathbb{R}$ estimates the infinite-horizon discounted return and provides the terminal bootstrap $V_\psi(s_{t+H})$. Together, the dynamics, reward, and value models evaluate and refine trajectories proposed by the generative prior through objective $J_H$.

%\Lina{Haldun, I do think we should mention the objector $J$ in this high level paragraph to tell people the key is 1) sample many action sequence trajectories using the trajectory generative model 2) improve and selectr actions based on $J$}

%Section~\ref{sec:inference} describes the planning procedure assuming these components are given: an iterative beam search that progressively commits prefix actions while refining the uncommitted suffix through repeated noising and denoising under the prior. Section~\ref{sec:online} then turns to the online setting, where the components must be learned from data the planner itself collects. We identify three challenges that arise from this coupling — slow value propagation under bootstrapping, mode collapse in the generative prior, and unreliable long-horizon planning under an inaccurate dynamics model —  which we address through targeted design choices. We additionally describe our exploration mechanism that leverages disagreement across the dynamics ensemble to drive the planner toward uncertain regions of the state space. Finally, Section~\ref{sec:training} presents the training objectives for each component.

Section~\ref{sec:inference} describes the planning procedure assuming these components are given, using beam search that progressively commits prefix actions while refining the remaining suffix via repeated noising and denoising under the prior. Section~\ref{sec:online} considers the online setting, where these components are learned from collected data, and addresses key challenges including slow value propagation under bootstrapping, mode collapse in the generative prior, and unreliable long-horizon planning under inaccurate dynamics. We also describe our exploration strategy based on disagreement across a dynamics ensemble. Finally, Section~\ref{sec:training} presents the training objectives for each component.

%\Lina{A major question hanging in my mind is: what is ultimate learning outcome? It would be more clear to first talk about the final outcome and how to use the final outcome for decision making. Then talk about how to learn them. I think you kind of follow this but inference (using the outcome to do decision making) and learning are still intervined together in the writing. It would be better to make the ``outcome'' more clear: you can do it in section 2. Section 3 will talk about how to learn. Section 2 is mainly introduce each element of your ``controller''} \HB{Rewrote to address this.}

\begin{figure}[t]
% \scalebox{0.73}{
\begin{minipage}{\linewidth}
\begin{algorithm}[H]
\caption{Generative Trajectory Planning with Beam Search}
\label{alg:gtp_inference}
\begin{algorithmic}[1]
\Require current state $s_t$, planning horizon $H$, beam size $B$, branching factor $D$, refinement steps $K_{\text{ref}}$, diffusion prior $p_\theta$, forward noising kernel $q_{K_{\text{ref}}\mid 0}$, objective functions $J_1, \ldots J_H$
\Statex \hrulefill
\Statex \textcolor{algcomment}{\textit{(0) Initialize beam:}}
\State Sample $B \times D$ candidate action sequences $\mathcal{C} \sim p_\theta(\cdot \mid s_t)$
\State $\mathcal{B}_0 \gets$ top-$B$ action sequences from $\mathcal{C}$ under $J_1(s_t, \cdot)$
\Statex \hrulefill
\For{$h=1$ to $H-1$}
    \Statex \textcolor{algcomment}{\textit{(1) Branch beam into candidate pool:}}
    \State Form $\mathcal{C}$ by replicating each $\mathbf{a} \in \mathcal{B}_{h-1}$ into $D$ copies, giving $|\mathcal{C}| = BD$
    \Statex \textcolor{algcomment}{\textit{(2) Refine future action sequences:}}
    \State Forward diffuse the candidate action sequences for $K_{\text{ref}}$ steps:
    $\mathcal{C}^{K_{\text{ref}}} \sim q_{K_{\text{ref}}\mid 0}\!\left(\cdot \mid \mathcal{C}\right)$
    \State Restore past actions $(t < h)$ with inpainting: $\mathcal{C}^{K_{\text{ref}}}_{[0:h)} \gets \mathcal{C}_{[0:h)}^{K_{\text{ref}}}$
    
    \For{$k=K_{\text{ref}}$ to $1$}
\State Reverse diffusion step:
$\mathcal{C}^{k-1} \gets p_\theta(\cdot \mid\mathcal{C}^{k}, s_t, k)$
    \State Restore past actions $(t< h)$ with inpainting: $\mathcal{C}^{k-1}_{[0:h)} \gets \mathcal{C}_{[0:h)}$
    \EndFor
    \Statex \textcolor{algcomment}{\textit{(3) Score and prune:}}
    \State $\mathcal{B}_h \gets$ top-$B$ candidates from $\mathcal{C}^0$ under $J_{h+1}(s_t, \cdot)$
\EndFor
\Statex \hrulefill
\State \Return $\arg\max_{\mathbf{a} \in \mathcal{B}_{H-1}} J_H(s_t, \mathbf{a})$
\end{algorithmic}
\end{algorithm}
\end{minipage}
\vspace{-10pt}
\end{figure}

\vspace{-5pt}
\subsection{Planning with Generative Prior}\label{sec:inference}
\vspace{-5pt}

Generating long-horizon trajectories is inherently challenging due to the curse of dimensionality: as the horizon grows, the space of action sequences expands exponentially, while high-reward sequences become vanishingly rare. Our generative prior restricts the search space to a lower-dimensional manifold of plausible trajectories; within this manifold, we apply an iterative search procedure that progressively refines candidates toward high-return regions using the learned models (see Fig.~\ref{fig:framework} for an overview). To make this concrete, we first define how the candidate sequences are scored.
% \vspace{-10pt}

\paragraph{Planning objective.} Given the current state $s_t$ and a candidate action prefix $\mathbf{a}_t = (a_t, \ldots, a_{t+h-1})$ of length $h \le H$, we score it by rolling the prefix forward through each dynamics model in the ensemble, accumulating discounted rewards along the way, and bootstrapping with the value function at the end of the prefix. For a single ensemble member $i$, this yields the $h$-step return
\begin{equation}
J_h^{(i)}(s_t, \mathbf{a}_t) = \sum_{h'=0}^{h-1} \gamma^{h'}\, r_\xi\!\left(\hat{s}^{(i)}_{t+h'}, a_{t+h'}\right) + \gamma^h\, V_\psi\!\left(\hat{s}^{(i)}_{t+h}\right),
\label{eq:per_dynamics_obj}
\end{equation}
where $\hat{s}^{(i)}_{t+h'+1}=f_{\phi_i}(\hat{s}^{(i)}_{t+h'}, a_{t+h'})$ is the rollout under the actions. We aggregate the per-dynamics estimates into a scalar via an aggregation rule $\mathrm{agg}$ (e.g., median, max; detailed in Sec.~\ref{sec:online}):
\begin{equation}
J^\mathrm{agg}_h(s_t, \mathbf{a}_t) = \mathrm{agg}_{i \in [N]} \, J_h^{(i)}(s_t, \mathbf{a}_t).
\label{eq:agg_obj}
\end{equation}

\paragraph{Beam-search procedure.}
We optimize the planning objective using beam search (Alg.~\ref{alg:gtp_inference}). We initialize candidate trajectories from the diffusion prior $p_\theta(\cdot\mid s_t)$ and retain the top $B$ under $J_1$. At each iteration, we commit one additional action, branch each candidate into $D$ variants, and refine their uncommitted suffixes through partial noising and denoising. The noise level $K_{\mathrm{ref}}$ controls the degree of refinement while preserving temporal coherence with the committed prefix. We then retain the top $B$ candidates under $J_{h+1}$. To balance local refinement and global exploration, we run multiple beams in parallel, allowing the planner to explore diverse modes of the generative prior.
\vspace{-10pt}
\subsection{Online Model Learning}
\vspace{-5pt}
\label{sec:online}
Section~\ref{sec:inference} described planning with fixed components. We now consider the online setting, where they are learned from data collected by the planner. We outline resulting challenges and the design choices that address them; the full loop is summarized in Algorithm~\ref{alg:gtp_training}.

\begin{figure}[t]
% \scalebox{0.73}{
\begin{minipage}{\linewidth}
\begin{algorithm}[H]
\caption{Online Training for Generative Trajectory Planning}
\label{alg:gtp_training}
\begin{algorithmic}[1]
\Require diffusion prior $p_\theta$, dynamics ensemble $\{f_{\phi_i}\}_{i=1}^N$, reward model $r_\xi$, value $V_\psi$ with target $V_{\bar\psi}$, temperature $\lambda$, target update rate $\tau$
\Statex \hrulefill
\State Initialize all parameters and replay buffer $\mathcal{D} \gets \emptyset$
\For{each environment step $t = 1, 2, \dots$}
    \Statex \textcolor{algcomment}{\textit{(1) Optimistic planning to act:}}
    \State Observe state $s_t$
    \State Execute $a^\star_t$ from $\mathbf{a}^\star_{t} \gets \text{Alg.~\ref{alg:gtp_inference}}(s_t;\, J_{1:H}^\mathrm{max})$, store observed transition $(r_t, s_{t+1}, d_t)$ to $\mathcal{D}$
    \Statex \textcolor{algcomment}{\textit{(2) Update dynamics and reward:}}
    \State Sample minibatch $(s, a, r, s_{t+1}) \sim \mathcal{D}$ and update $\{f_{\phi_i}\}$, $r_\xi$ via supervised learning (App.~\ref{app:networks-training})
    \Statex \textcolor{algcomment}{\textit{(3) Conservative planning to construct learning targets:}}
    \State For each $s$ in minibatch, plan top-$B$ trajectories $\{\mathbf{a}^{b}\}_{b=1}^B \gets$ Alg.~\ref{alg:gtp_inference}$(s;\, J_{1:H}^\mathrm{median})$
    \State Augment elite set $\{\mathbf{a}^{b}\}_{b=1}^B$ with Langevin samples (App.~\ref{app:langevin})
    \Statex \textcolor{algcomment}{\textit{(4) Update value and prior:}}
    \State Update $V_\psi$, $p_\theta$ as in Section~\ref{sec:training}, $\lambda$ as in App.~\ref{app:networks-training}; and target $V_{\bar\psi}$ by Polyak averaging.
\EndFor
\end{algorithmic}
\end{algorithm}
\end{minipage}
\vspace{-20pt}
\end{figure}

\paragraph{Mode collapse during prior training.}
Training our action prior in a sample-and-imitate loop is prone to collapse: as $p_\theta$ sharpens, the elite set drawn from its own samples differs little from what it already produces, so retraining reinforces the existing distribution rather than improving it. To inject novel high-scoring candidates, we perform a few steps of Langevin ascent on $J_H$ initialized from the planner's elites, perturbing them along the planning objective (cf. Appendix~\ref{app:langevin}). 

% \vspace{-10pt}
\paragraph{Exploration under learned dynamics.}
Unlike uncoordinated exploration through action noise, GTP leverages its dynamics ensemble to guide exploration over entire trajectories. Each ensemble member $f_{\phi_i}$ independently evaluates candidate sequences, with disagreement among predicted returns $\{J_h^{(i)}\}_{i=1}^N$ capturing trajectory-level uncertainty. When acting, we use \emph{max} aggregation in Eq.~\ref{eq:agg_obj}, favoring trajectories predicted to achieve high returns under at least one ensemble member. This induces optimism under uncertainty and encourages structured, temporally coherent exploration.
% \vspace{-10pt}

\paragraph{Slow value propagation in long horizons.}
Standard temporal-difference learning propagates rewards slowly over long horizons. Instead, we leverage multi-step return estimates from beam search, weighting different horizons by inverse ensemble variance to balance short-horizon bias against long-horizon uncertainty. We use \emph{median} aggregation across ensemble members for robust value targets, in contrast to the optimistic \emph{max} aggregation used for exploration. This allows GTP to learn from reliable trajectory-level estimates while encouraging exploration through optimistic planning (Sec.~\ref{sec:training}).

\paragraph{Computational efficiency.}
Although GTP incurs additional computational overhead from planning, this cost can be substantially reduced during both training and inference (cf. Appendix~\ref{app:timing}). During training, reusing planner-generated elite trajectories across multiple gradient updates reduces update costs without sacrificing performance (cf. Appendix~\ref{app:bootstrap}). At inference, the learned generative prior can generate actions without explicit planning, while executing multiple actions between replanning steps further reduces computational overhead (cf. Appendix~\ref{app:efficient_inference}). These strategies enable more efficient training and deployment with minimal performance degradation.

\vspace{-8pt}
\subsection{Training Objectives}\label{sec:training}
\vspace{-8pt}

All learned components are trained jointly from a shared replay buffer $\mathcal{D}$. The dynamics ensemble $\{f_{\phi_i}\}_{i=1}^N$ and reward model $r_\xi$ are trained by supervised regression, with each ensemble member receiving an independent bootstrap resample of the minibatch to encourage diversity. A termination head trained jointly with the reward modulates accumulated returns and the terminal value bootstrap during planning, so rollouts correctly account for early termination. Explicit losses appear in Appendix~\ref{app:networks-training}; the value and generative-prior losses, central to our framework, are described next.

\paragraph{Value learning.}
Following Section~\ref{sec:online}, we replace one-step TD targets with trajectory-level estimates derived from the planner. For each state $s$ in the batch, we run Algorithm~\ref{alg:gtp_inference} to obtain the top-scoring trajectory $\mathbf{a}^\star$ and define a family of $h$-step targets:
\begin{equation*}
\bar{J}^\mathrm{med}_h(s)
=
\mathrm{med}_{i \in [N]}
\, J_h^{(i)}\!\left(s, \mathbf{a}^\star\right),
\qquad h=1,\ldots,H,
\end{equation*}
where the median across ensemble members mitigates overestimation from individual optimistic rollouts. Short horizons are biased by bootstrapping while long horizons accumulate model error; we balance this tradeoff via inverse-variance weights based on cross-ensemble disagreement:
\begin{equation*}
\omega_h(s)
=
\mathrm{softmax}_h\!\left(1/\sigma_h^2(s)\right),
\qquad
\sigma_h^2(s)
=
\mathrm{Var}_{i \in [N]}
\, {J}_h^{(i)}(s,\mathbf{a}^\star),
\end{equation*}
following~\citep{steve-uncertainty} with an added softmax to sharpen the contribution of reliable horizons. The value function is then trained to regress onto the weighted target:
\begin{equation}
\label{eq:value_loss}
\mathcal{L}_{\text{val}}(\psi)
=
\mathbb{E}_{s\sim\mathcal D}
\!\left[
\mathrm{Huber}
\!\left(
V_\psi(s),
\sum_{h=1}^{H}
\omega_h(s)\,
\bar{J}^{\mathrm{med}}_h(s)
\right)
\right].
\end{equation}
We employ the standard target-network and double-value mechanisms used in modern TD learning.

\paragraph{Generative prior learning.}
The diffusion prior is trained to match a target distribution over action sequences induced by trajectory-level evaluations. For each $s$ in the batch, we run beam search to obtain $B$ high-scoring trajectories, augment them with Langevin samples as described in Section~\ref{sec:online}, and take the top-$M$ from the augmented pool as the elite set. Their planner scores $\{J^{m}_H\}_{m=1}^M$ are converted into normalized weights $w^{m}$, and the prior is trained with:
\begin{equation}\label{eq:prior_loss}
\mathcal{L}_{\text{prior}}(\theta) = \mathbb{E}_{s\sim\mathcal{D},\, k,\, \epsilon}\!\left[\, \sum_{m=1}^M w^{m}\, \big\| \epsilon - \epsilon_\theta(\mathbf{a}^{m, k}, s, k) \big\|^2 \right], \; w^{m}(s, \mathbf{a}^{m})=\mathrm{softmax}_m\!\left(J^m_H/\lambda\right)
\end{equation}
where $\mathbf{a}^{m, k}$ is the $m$-th elite trajectory corrupted to diffusion step $k$. The exponential reweighting is analogous to MPPI~\citep{williams2017information}, the main difference being that we amortize this update across states by distilling it into a single generative model. We additionally adapt $\lambda$ online to maintain a target effective sample size~\citep{ess-ref} over the elite weights, avoiding the fragility of a fixed temperature as returns concentrate during training (see Appendix~\ref{app:networks-training}).

\vspace{-5pt}
\subsection{Theoretical Analysis}\label{sec:theory}
\vspace{-5pt}

We analyze the convergence of GTP's generative prior under idealized conditions. Let $p_n(\mathbf{a}\mid s)$ denote the prior after $n$ updates. For a fixed state $s$ and planning objective $J_H(s,\mathbf{a})$, we assume the learned models remain fixed. In the limit of sufficient trajectory sampling and accurate diffusion model fitting, the prior update corresponds to a KL-regularized policy improvement step over action sequences satisfying:
\vspace{-5pt}
\begin{equation}
    p_{n+1}(\mathbf{a}\mid s)
    \propto p_n(\mathbf{a}\mid s)
    \exp\left(\frac{J_H(s,\mathbf{a})}{\lambda}\right).
\end{equation}
We show that these updates improve the prior and converge toward optimal trajectories.
\begin{theorem}[Convergence of the Generative Prior]
\label{thm:prior_convergence}
Suppose $J_H(s,\mathbf{a})$ is bounded and the prior
follows the idealized update with a fixed temperature
$\lambda>0$. Then, each update satisfies:
\begin{equation}
    \mathbb{E}_{p_{n+1}}[J_H]
    \geq
    \mathbb{E}_{p_n}[J_H]
    + \lambda D_{\mathrm{KL}}(p_{n+1}\|p_n).
\end{equation}
Furthermore, if the initial prior assigns positive probability
to the set of $\delta$-optimal action sequences for every
$\delta>0$, then $\mathbb{E}_{p_n}[J_H] \to \sup_{\mathbf{a}} J_H(s,\mathbf{a})$, with the prior assigning asymptotically all probability mass to $\delta$-optimal sequences for every $\delta>0$.
\end{theorem}

This result establishes that iterative prior updates
monotonically improve the expected planning objective and
asymptotically concentrate the generative prior around
optimal trajectories. In practice, GTP employs beam search
to generate improved trajectory proposals, adaptive
temperature selection to control the prior update, and
Langevin augmentation to mitigate mode collapse.
We discuss how these components relate to the theoretical
assumptions and provide formal proofs in
Appendix~\ref{app:theory}.

\vspace{-5pt}
\section{Experiments}\label{sec:experiments}
\vspace{-5pt}
In this section, we present experiments designed to answer three questions.
\textbf{(Q1)} Does combining a diffusion action prior with sample-based planning yield competitive performance on standard continuous-control benchmarks?
\textbf{(Q2)} Does the proposed approach scale to long-horizon manipulation and high-dimensional image observations?
\textbf{(Q3)} Which design choices drive GTP's performance, and how does its exploration behavior differ from existing methods?

\begin{figure}[t]
    \centering
    \includegraphics[width=\linewidth]{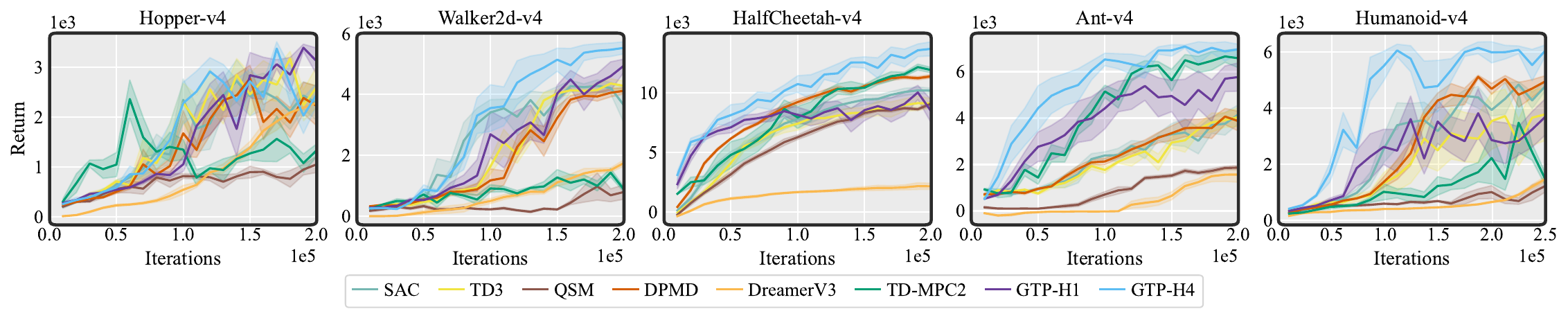}
    \caption{\textbf{GTP consistently outperforms baselines.} Evaluation return curves on OpenAI Gym MuJoCo environments. Solid lines show the mean and shaded regions the standard error over 5 random seeds.}
    \label{fig:mujoco}
    \vspace{-15pt}
\end{figure}

\begin{table}[t]
\centering
\caption{\textbf{Results on locomotion benchmark.} Performance on OpenAI Gym MuJoCo environments. We report
mean returns and standard errors over 5 random seeds, each trained
for 1M environment interactions. }
\resizebox{0.9\linewidth}{!}{%
\begin{tabular}{cl|ccccc}
\toprule
 & Algorithm & Hopper-v4 & Walker2d-v4 & HalfCheetah-v4 & Ant-v4 & Humanoid-v4 \\
\midrule
\textit{Classic} & SAC & 3351 $\pm$ 39 & 4453 $\pm$ 158 & 10515 $\pm$ 286 & 4229 $\pm$ 366 & 5118 $\pm$ 153 \\
\textit{Model-Free RL} & TD3 & 3314 $\pm$ 47 & 4408 $\pm$ 141 & 9301 $\pm$ 546 & 4278 $\pm$ 293 & 3852 $\pm$ 967 \\
\midrule
\textit{Diffusion} & QSM & 1146 $\pm$ 133 & 839 $\pm$ 212 & 9084 $\pm$ 179 & 1868 $\pm$ 97 & 1294 $\pm$ 207 \\
\textit{Policy RL} & DPMD & 3147 $\pm$ 142 & 4133 $\pm$ 109 & 11591 $\pm$ 123 & 4092 $\pm$ 354 & 5269 $\pm$ 44 \\
\midrule
\textit{Trajectory} & CEM-MPC & 141 $\pm$ 37 & 71 $\pm$ 31 & 651 $\pm$ 60 & 379 $\pm$ 53 & 318 $\pm$ 17 \\
\textit{Optimization} & CEM-MPC-V & 2437 $\pm$ 345 & 288 $\pm$ 54 & 7293 $\pm$ 211 & 5103 $\pm$ 811 & 4396 $\pm$ 569 \\
\midrule
\textit{Model-Based} & DreamerV3 & 3031 $\pm$ 65 & 2961 $\pm$ 140 & 2251 $\pm$ 323 & 2279 $\pm$ 365 & 879 $\pm$ 109 \\
\textit{RL} & TD-MPC2 & 2631 $\pm$ 349 & 1622 $\pm$ 148 & 12216 $\pm$ 267 & 7339 $\pm$ 119 & 4160 $\pm$ 795 \\
\midrule
\multirow[c]{2}{*}{Ours} & GTP-H1 & 3398 $\pm$ 71 & 5007 $\pm$ 212 & 10372 $\pm$ 439 & 6207 $\pm$ 418 & 4596 $\pm$ 785 \\
\textit{} & GTP-H4 & \textbf{3525 $\pm$ 49} & \textbf{5546 $\pm$ 201} & \textbf{13944 $\pm$ 498} & \textbf{7369 $\pm$ 73} & \textbf{6660 $\pm$ 132} \\
\bottomrule
\end{tabular}
}
\label{tab:mujoco}
\vspace{-10pt}
\end{table}

We address (Q1) on MuJoCo locomotion (Section~\ref{sec:mujoco});
(Q2) on three state-based manipulation tasks (Fetch) and a vision-based Push-T task (Section~\ref{sec:manipulation}), where the latter demonstrates scalability to high-dimensional inputs; and (Q3) through component ablations, a study of the planning horizon, and an analysis of exploration behavior (Section~\ref{sec:exp_ablations}).
Additionally, we discuss a failure case and its remedy in Appendix~\ref{sec:aggregation_failure}, provide a sensitivity analysis of beam search parameters in Appendix~\ref{app:beam_ablation}, and analyze computational costs, efficien training, and efficient inference in Appendices~\ref{app:timing}, \ref{app:bootstrap}, and~\ref{app:efficient_inference}, respectively.
All methods are implemented in JAX, with hyperparameters detailed in Appendix~\ref{app:networks-training}.

\paragraph{Baselines.}
We compare GTP against baselines spanning four families.
\textit{Model-free RL:} SAC \citep{haarnoja2018soft} and
TD3 \citep{fujimoto2018addressing}, two strong off-policy actor-critic
methods.
\textit{Diffusion policy RL:} QSM \citep{psenka2023learning} and
DPMD \citep{ma2025efficient}, which train diffusion policies with
RL objectives but act without planning at inference.
\textit{Trajectory optimization:} CEM-MPC and its value-augmented
variant CEM-MPC-V, which adds a learned terminal value to the
planning objective. Rather than training separate models, both
variants reuse GTP's learned dynamics and (for CEM-MPC-V)
value function, and replace the diffusion prior with Gaussian CEM
sampling, in similar spirit to~\citep{chua2018deep}. This isolates the contribution of the
diffusion prior while holding the rest of the planning stack fixed.
\textit{Model-based RL:} DreamerV3 \citep{hafner2023mastering} and
TD-MPC2 \citep{hansen2023td}, the current state of the art in
model-based RL. All baselines are run under the same evaluation protocol; we detail
each implementation, its source, and any modifications in
Appendix~\ref{app:baselines}.

\vspace{-10pt}
\subsection{Locomotion Tasks}\label{sec:mujoco}
\vspace{-5pt}

\paragraph{Setup.} All methods are trained for 1M environment interactions using 5 parallel environments (4 for Humanoid). We report mean $\pm$ standard error over 5 runs. GTP is evaluated at horizons $H=1$ and $H=4$ to illustrate the benefit of longer planning.

\begin{figure}[t]
    \centering
    \includegraphics[width=\linewidth]{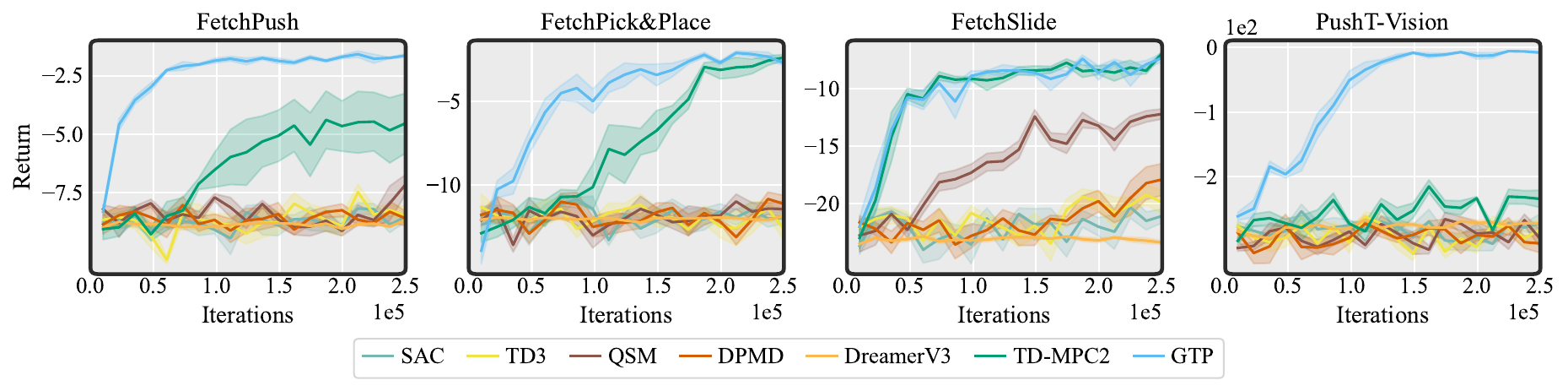}
    \caption{\textbf{GTP improves sample efficiency in long-horizon settings.} Evaluation return curves on Fetch and PushT environments. Solid lines show the mean and shaded regions the standard error over 5 random seeds.}
    \vspace{-15pt}
    \label{fig:fetch}
\end{figure}
\begin{table}[t]
\centering
\caption{\textbf{Long-horizon manipulation results.} Success rates (\%) on long-horizon manipulation tasks.
Mean and standard error over 5 random seeds, each trained for
1M environment interactions.}
\resizebox{0.8\linewidth}{!}{%
\begin{tabular}{ll|cccc}
\toprule
 & Algorithm & FetchPush & FetchPick\&Place & FetchSlide & PushT-Vision \\
\midrule
\textit{Classic} & SAC & 6 $\pm$ 0 & 3 $\pm$ 0 & 1 $\pm$ 1 & 0 $\pm$ 0 \\
\textit{Model-Free RL} & TD3 & 6 $\pm$ 0 & 3 $\pm$ 0 & 5 $\pm$ 1 & 0 $\pm$ 0 \\
\midrule
\textit{Diffusion} & QSM & 12 $\pm$ 3 & 4 $\pm$ 0 & 14 $\pm$ 2 & 0 $\pm$ 0 \\
\textit{Policy RL} & DPMD & 6 $\pm$ 0 & 3 $\pm$ 0 & 5 $\pm$ 1 & 0 $\pm$ 0 \\
\midrule
\textit{Trajectory} & CEM-MPC & 40 $\pm$ 4 & 33 $\pm$ 5 & 50 $\pm$ 3 & 16 $\pm$ 2 \\
\textit{Optimization} & CEM-MPC-V & 78 $\pm$ 6 & 68 $\pm$ 5 & 67 $\pm$ 1 & 33 $\pm$ 4 \\
\midrule
\textit{Model-Based} & DreamerV3 & 7 $\pm$ 1 & 1 $\pm$ 0 & 0 $\pm$ 0 & 0 $\pm$ 0 \\
\textit{RL} & TD-MPC2 & 67 $\pm$ 17 & \textbf{99 $\pm$ 0} & 70 $\pm$ 3 & 0 $\pm$ 0 \\
\midrule
\textit{Ours} & GTP & \textbf{100 $\pm$ 0} & 98 $\pm$ 1 & \textbf{76 $\pm$ 3} & \textbf{99 $\pm$ 1} \\
\bottomrule
\end{tabular}
}
\label{tab:fetch}
\vspace{-15pt}
\end{table}

\paragraph{Results.}
Table~\ref{tab:mujoco} and Fig.~\ref{fig:mujoco} report performance and training curves. GTP-H4 outperforms all baselines across the five MuJoCo environments. While GTP-H1 is competitive with diffusion-based RL methods like DPMD, increasing the horizon to $H=4$ yields substantially stronger performance, confirming the value of trajectory-level planning. Trajectory optimization baselines trail by a wide margin even when reusing GTP's dynamics and value, indicating that sampling-based planning with a Gaussian prior scales poorly here and highlighting the benefit of the diffusion prior. GTP also outperforms strong model-based baselines like TD-MPC2 and DreamerV3, suggesting trajectory-level reasoning has advantages over step-by-step unrolling. GTP-H4 additionally attains high returns faster than all baselines except on Hopper (Fig.~\ref{fig:mujoco}), demonstrating strong sample efficiency. CEM and TD-MPC2 are sensitive to termination conditions, since Gaussian particle rollouts blur termination signals by mixing terminating with non-terminating trajectories; this is most evident on Walker2d, where GTP is unaffected. Together, these results answer \textbf{(Q1)} affirmatively: GTP matches or exceeds the strongest model-free, diffusion-policy, and model-based baselines on standard locomotion benchmarks.

\vspace{-10pt}
\subsection{Long-horizon Manipulation Tasks}\label{sec:manipulation}
\vspace{-5pt}

\paragraph{Setup.}
All methods are trained for 1M environment interactions, collecting
data from 4 parallel environments. We run 5 independent seeds per
method and evaluate the final model on 100 independent trials per
seed, using the same set of evaluation seeds across methods. We
report mean success rate and standard error. GTP is run at planning
horizon $H=8$.

\paragraph{Fetch.}
Table~\ref{tab:fetch} and Fig.~\ref{fig:fetch} summarize performance across the three Fetch tasks. Model-free and diffusion-policy baselines (SAC, TD3, QSM, DPMD) consistently fail with near-zero success rates, confirming that reactive policies are insufficient for long-horizon, contact-rich manipulation. TD-MPC2 is a strong planning-based baseline but unstable: on FetchPush it fails to converge in 2 of 5 seeds, whereas GTP converges reliably in all runs, indicating a clear robustness and sample-efficiency advantage. The same gap appears on FetchPickAndPlace (Fig.~\ref{fig:fetch}), where GTP reaches high reward substantially faster despite similar final success rates. CEM-MPC-V is competitive and improves substantially over vanilla CEM-MPC, confirming the importance of a learned terminal value for sampling-based planning, but still trails GTP and shows higher variance, suggesting that value-guided Gaussian proposals remain limited for exploring complex action sequences. Overall, GTP achieves higher final performance with faster convergence and lower variance than both TD-MPC2 and CEM-based planners, pointing to the generative trajectory prior as the key driver: structured generative proposals are more effective than simple parametric distributions for long-horizon manipulation.

\paragraph{PushT-Vision.}
PushT-Vision is a visual manipulation task built on~\citet{florence2022implicit} (details in Appendix~\ref{app:vision}). As shown in Table~\ref{tab:fetch} and Fig.~\ref{fig:fetch}, GTP achieves near-perfect success while all learning-based baselines completely fail in the default training regime, exposing the limitations of existing methods in long-horizon tasks with high-dimensional observations. Among planning approaches, CEM-MPC-V improves over vanilla CEM-MPC but remains far below GTP, again suggesting that value-guided Gaussian sampling is insufficient in high-dimensional, visually grounded settings. Scaling data alone does not close the gap: extending training by 2M additional environment interactions leaves all baselines unchanged except TD-MPC2, which eventually reaches near-optimal performance but requires substantially more data than GTP (Appendix~\ref{app:pusht}). The early-training success of GTP highlights the sample-efficiency advantage of the generative trajectory prior, which enables effective exploration over trajectories. Combined with the Fetch results, this answers \textbf{(Q2)}: the diffusion-prior planner extends consistently from proprioceptive control to long-horizon manipulation and further to image-based decision making, where its advantages become even more pronounced.

\begin{figure}[t]
\centering
    \begin{minipage}{0.35\linewidth}
    \centering
    \resizebox{\linewidth}{!}{
    \begin{tabular}{llc}
    \toprule
    Category & Variant &  Return \\
    \midrule
    \multirow{2}{*}{Augment.}
    & \lkey{c8} No aug. & 1652 $\pm$ 470 \\
    & \lkey{c7} Gauss. aug. & 4599 $\pm$ 529 \\
    \midrule
    \multirow{2}{*}{Planning}
    & \lkey{c6} TD-1 & 4556 $\pm$ 282 \\
    & \lkey{c5} No plan. & 3177 $\pm$ 782 \\
    \midrule
    \multirow{3}{*}{Ensemble}
    & \lkey{c4} No ensemble & 1998 $\pm$ 475 \\
    & \lkey{c3} No cons. update & 3567 $\pm$ 489 \\
    & \lkey{c2} No opt. rollout & 5238 $\pm$ 206 \\
    \midrule
    Full & \lkey{c1} GTP & 5546 $\pm$ 201 \\
    \bottomrule
    \end{tabular}
    }
    \end{minipage}
    \begin{minipage}{0.48\linewidth}
    \centering
    \includegraphics[width=\linewidth]{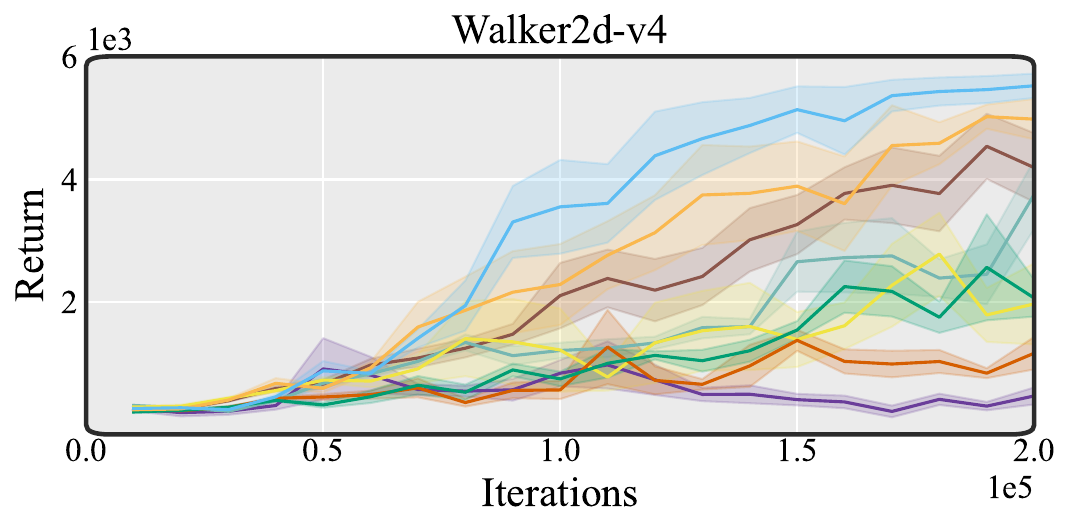}
    \end{minipage}
\caption{ \textbf{Ablations of key components.} Best mean returns and standard errors over 5 random seeds.}
\label{fig:extra-ablations}
\vspace{-10pt}
\end{figure}

\begin{figure}[t]
    \centering
    \begin{minipage}{0.38\linewidth}
    \includegraphics[width=\linewidth]{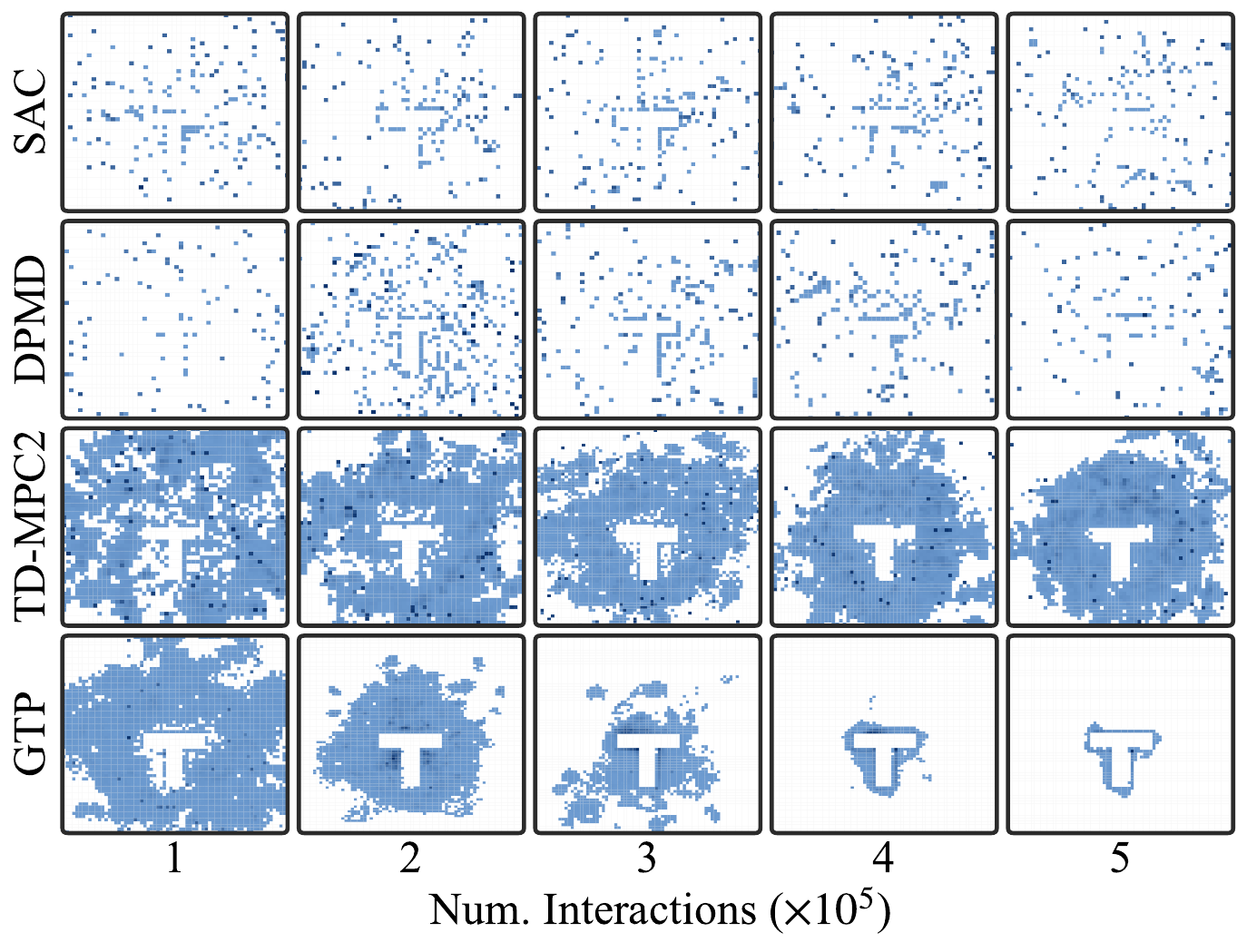}
        \caption{\textbf{Exploration on PushT-Vision.} End-effector visitation heatmaps over training in an object-centric frame.}

    \label{fig:heatmaps}
    \end{minipage}
    \hspace{10pt}
    \begin{minipage}{0.4\linewidth}
    \includegraphics[width=\linewidth]{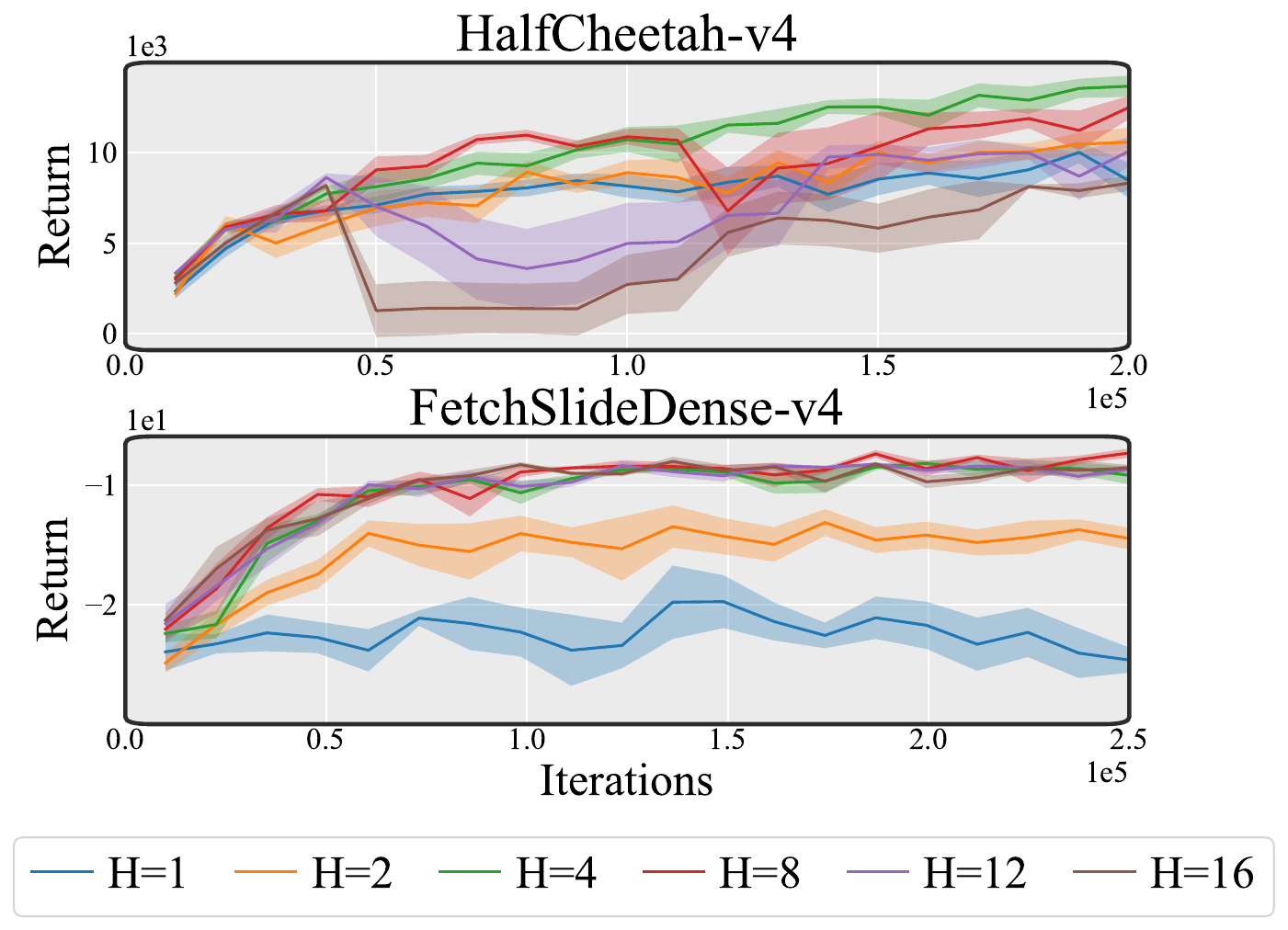}
    \caption{ \textbf{Effect of planning horizon.} Evaluation returns under varying horizon.}
    \label{fig:horizon-ablation}
    \end{minipage}
\vspace{-10pt}
\end{figure}

\vspace{-5pt}
\subsection{Analysis}\label{sec:exp_ablations}
\vspace{-5pt}
Having shown that GTP outperforms strong baselines across our benchmarks, we now examine what drives these gains, ablating its main components, studying the effect of the planning horizon, and analyzing how long-horizon planning shapes exploration during training.

\paragraph{Candidate-set augmentation.}
Removing the Langevin augmentation (\textit{No Aug.}) causes mode collapse: with no mechanism to inject novel candidates into the elite set, the prior concentrates on its current mode and the regression target degenerates into a fixed point. Replacing Langevin with the isotropic Gaussian perturbation used in prior online diffusion RL methods~\citep{wang2024diffusion, ma2025efficient} (\textit{Gauss. Aug.}) recovers much of the lost performance but still trails GTP, consistent with our arguments in Appendix~\ref{app:langevin} that uninformed per-timestep noise scales poorly with horizon.

\paragraph{Planning and value learning.}
Replacing trajectory-level value targets with single-step TD bootstrapping (\textit{TD-1}) yields lower final returns due to having slow long-horizon information propagation that motivated trajectory-level targets in the first place. Removing iterative beam search in favor of a best-of-$BD$ filter over a single batch of samples (\textit{No Plan.}) drops performance substantially, showing that filtering is not a substitute for our refinement procedure in Alg.~\ref{alg:gtp_inference}.

\paragraph{Ensemble and aggregation.}
Collapsing the dynamics ensemble to a single model (\textit{No Ensemble}) is catastrophic, removing the uncertainty signal the rest of the planner relies on. Switching the update aggregation away from the conservative scheme (\textit{No cons. update}) destabilizes value learning: without down-weighting ensemble disagreement, optimistic predictions from individual members leak into the value targets and amplify model bias. Switching the rollout aggregation away from the optimistic scheme (\textit{No opt. rollout}) is more subtle: final return is competitive with GTP, but the training curve in Fig.~\ref{fig:extra-ablations} sits below GTP throughout most of training, indicating that the conservative rollout sacrifices sample efficiency by suppressing exploration of high-uncertainty, potentially high-return sequences.

\paragraph{Planning horizon.}
We sweep $H \in \{1, 2, 4, 8, 12, 16\}$ on HalfCheetah and FetchSlide, training a separate model at each horizon (Fig.~\ref{fig:horizon-ablation}). Increasing the horizon up to $H=4$ yields clear gains on both tasks, confirming the benefit of trajectory-level planning. Beyond $H=4$, behavior diverges: on HalfCheetah, performance degrades as the horizon grows, which we attribute to the difficulty of accurately rolling out learned dynamics in a contact-rich locomotion task; on FetchSlide, performance plateaus rather than degrades, indicating that longer horizons are tolerable when dynamics are simpler. GTP trains stably even at $H=16$ on manipulation tasks, suggesting that the limiting factor at long horizons is dynamics-model accuracy rather than training stability. We hypothesize that an additional benefit of longer horizons is improved exploration. By optimizing entire future trajectories, the planner can reason about delayed consequences of actions and guide data collection toward states that are informative for eventual task completion.

\paragraph{Exploration analysis.} To visualize the exploration behavior induced by different learning algorithms, we analyze the distribution of visited states throughout training on \texttt{PushT}. At several checkpoints, we collect 2000 evaluation rollouts and record the agent end-effector positions. To isolate exploration relative to the object rather than its absolute location, we transform all states into an object-centric coordinate frame by translating the T-shaped block to the center of the workspace. Figure~\ref{fig:heatmaps} then shows a heatmap of the resulting end-effector visitation distribution over the course of training.

The visitation patterns highlight the importance of temporally coherent exploration. SAC and DPMD rarely interact with the object, while TD-MPC2 explores broadly but indiscriminately despite using model-based planning. This suggests that planning over policy rollouts with Gaussian perturbations is insufficient to discover coordinated behaviors, as independently perturbed actions fail to capture temporal dependencies. In contrast, GTP jointly models action sequences, enabling structured exploration that initially discovers useful interactions and progressively concentrates around the object’s contour and key contact regions. This demonstrates how generative trajectory priors can guide exploration toward task-relevant behaviors rather than merely increasing state coverage.

These findings answer (Q3): each design choice contributes to GTP’s performance, with the exploration analysis highlighting its ability to discover directed, task-relevant behaviors rather than diffuse state-space coverage. More broadly, our experiments demonstrate that combining expressive generative trajectory priors with model-based planning yields strong sample efficiency and final performance across diverse locomotion and manipulation tasks. Together, these results support our central hypothesis that generative models can serve as effective proposal distributions for planning, enabling structured exploration and long-horizon reasoning while improving online decision making beyond reactive policies and conventional trajectory optimization.

\vspace{-8pt}
\section{Conclusion}\label{sec:conclusion}
\vspace{-8pt}

We introduced GTP, a model-based framework for online decision making that uses a learned generative prior to propose trajectories for planning. By combining diffusion-based trajectory generation with model-based evaluation and iterative refinement, GTP enables explicit trajectory-level reasoning without optimizing action sequences from scratch at every step. Crucially, planning provides a learning signal, as trajectories discovered through search are distilled into the generative prior, allowing the agent to improve its proposals through interaction. Across continuous control and manipulation benchmarks, GTP achieves strong sample efficiency and performance. Our ablations highlight the importance of iterative refinement, trajectory-level value learning, and structured exploration. Together, these findings highlight the importance of combining expressive generative priors with model-based planning to enable effective long-horizon reasoning and online learning.

Our results suggest several directions for future work. First, reducing reliance on terminal value functions could enable more flexible long-horizon planning, particularly in settings with sparse rewards or unreliable value estimates. Developing alternative approaches to evaluating long-term outcomes could allow GTP to plan over longer horizons without depending heavily on accurate value predictions. Second, extending GTP to more complex tasks with rich visual observations presents an interesting direction for future research. Finally, developing more efficient and effective planning algorithms could improve the quality and diversity of generated trajectories while reducing computational overhead, allowing greater use of inference-time computation to discover behaviors beyond those represented by the current generative prior.

\vspace{-8pt}
\section*{Acknowledgements}
\vspace{-9pt}
This work was supported in part by the NSF AI Institute under Grant Numbers~2112085, 2624398, NSF ECCS-2401390, and ONR Grant N000142512173. The authors also acknowledge support from the Kempner Institute, Pickle Robotics, Amazon AGI Labs, Meta Reality Labs, and Amazon Robotics. The authors thank Weirui Ye for insightful discussions.

\bibliography{neurips_2026}
\bibliographystyle{abbrvnat}

\newpage
%%%%%%%%%%%%%%%%%%%%%%%%%%%%%%%%%%%%%%%%%%%%%%%%%%%%%%%%%%%%

\appendix

\section*{\Large \textbf{Appendix}}

In this appendix, we provide additional implementation details, experimental analyses, and theoretical results to complement the main text. In Appendix~\ref{app:related_work}, we review related work on reinforcement learning and diffusion-based decision making. In Appendix~\ref{app:networks}, we describe network architectures, training procedures, and hyperparameter settings. In Appendix~\ref{app:langevin}, we detail our Langevin-based trajectory augmentation strategy. In Appendix~\ref{sec:aggregation_failure}, we analyze ensemble aggregation strategies and their failure modes. In Appendix~\ref{app:beam_ablation}, we study the sensitivity of GTP to beam search parameters. In Appendix~\ref{app:pusht}, we provide additional environment details and extended training results for PushT-Vision. In Appendix~\ref{app:baselines}, we describe the implementation of baseline methods. In Appendix~\ref{app:vision}, we detail our approach to handling image observations. In Appendix~\ref{app:timing}, we analyze computational costs during training and inference. In Appendix~\ref{app:bootstrap}, we introduce trajectory reuse during training to reduce computational overhead. In Appendix~\ref{app:efficient_inference}, we investigate strategies for reducing inference-time computational costs through planning-free action generation and multi-step execution. Finally, in Appendix~\ref{app:theory}, we present theoretical convergence guarantees for the generative prior, together with proofs and a discussion of the underlying assumptions.

\vspace{-10pt}
\section{Related Work}\label{app:related_work}
\vspace{-5pt}

\paragraph{Reinforcement learning.}
Reinforcement learning has achieved strong empirical success in continuous control through both model-free and model-based approaches. Model-free methods such as SAC~\cite{haarnoja2018soft} and TD3~\cite{fujimoto2018addressing} learn reactive policies by optimizing value functions, but rely on short-horizon bootstrapping, which can limit long-horizon reasoning. Model-based methods aim to address this limitation by learning predictive models of the environment and incorporating them into decision making. Early approaches such as PETS~\cite{chua2018deep} combine learned dynamics with sampling-based planning, enabling explicit reasoning over action sequences but requiring optimization from scratch at each step. More recent methods integrate model learning with value estimation and policy optimization. DreamerV3~\cite{hafner2023mastering} learns a latent world model and trains policies and value functions entirely in imagination, using model rollouts to improve sample efficiency. TD-MPC2~\cite{hansen2023td} combines learned latent dynamics with short-horizon planning and value bootstrapping, achieving strong performance by tightly coupling model predictive control with reinforcement learning. Despite these advances, decision making in these methods remains largely policy-centric. Actions are selected step-by-step using learned value functions or short-horizon planning with bootstrapping, rather than reasoning over entire trajectories. Model-based rollouts interleave dynamics predictions with policy or value evaluations at each step, which limits the ability to capture coherent, trajectory-level structure. Exploration is likewise driven by local perturbations, typically Gaussian noise in action space, rather than structured variation over action sequences. Our approach addresses these limitations by shifting from step-wise decision making to trajectory-level reasoning through our learned generative prior.

\paragraph{Diffusion models for online decision-making.}
A growing body of work uses diffusion models as the policy class for online RL. \citet{psenka2023learning} dispense with an explicit policy network entirely and sample actions by running Langevin dynamics on the gradient of a learned critic, treating the Q-function as an unnormalized log-density over actions. While elegant, this approach inherits the variance and conditioning issues of critic gradients and offers no learned proposal to amortize sampling cost. \citet{wang2024diffusion} instead train a dedicated diffusion policy and update it by backpropagating through the full denoising chain; this enables expressive multimodal action distributions but is computationally expensive and, in their reported results, fails to translate the added capacity into reliable performance gains. More recent methods~\citep{ding2024diffusion, ma2025efficient} sidestep this bottleneck by training the diffusion policy with denoising score-matching objectives derived from the critic, decoupling policy learning from the sampling chain and yielding strong results on standard continuous-control benchmarks.

Despite their use of diffusion, these methods retain the structure of standard off-policy actor-critic algorithms: the policy is updated to follow the critic, and exploration is left to the environment-step noise injected using Gaussian perturbations. There is no explicit mechanism for proposing trajectories the current policy is unlikely to generate, no use of the model class to plan over multiple timesteps, and no objective-aware perturbation of the proposal distribution. As a consequence, the practical benefit of using a diffusion policy reduces to a more expressive single-step action distribution, leaving the broader generative capacity of diffusion largely unexploited. Our work targets exactly this gap by treating the diffusion model as a learned proposal over action \emph{sequences}, refined online with planning through learned models.

\paragraph{Diffusion models for offline decision-making and imitation.}
A parallel line of work uses diffusion models to generate \emph{trajectories} rather than single actions, exploiting the multi-step structure of control. \citet{janner2022planning} introduce Diffuser, which trains an unconditional diffusion model over state-action sequences offline and steers sampling at inference time with classifier-style guidance from a learned reward. \citet{ajay2023conditional} refine this idea with classifier-free conditioning, showing that returns, goals, and constraints can be incorporated directly into the generative model and that the resulting planner is competitive with strong offline RL baselines. \citet{chi2023diffusion} demonstrate that diffusion policies trained over short action chunks dramatically outperform Gaussian and energy-based alternatives on real-world manipulation, and that the gains come precisely from the model's ability to produce temporally coherent multi-step actions. Collectively, these results establish that diffusion is a strong fit for sequence-level decision making when supervision is plentiful, whether through offline datasets or expert demonstrations. Our setting differs in that supervision is generated online from the planner itself, but the same property, that diffusion models excel at producing coherent action sequences, is what motivates our use of a sequence-level prior over the single-action policies common in online decision-making.

\section{Network Architectures And Hyperparameters}\label{app:networks}

\subsection{Networks and Training}\label{app:networks-training}
The diffusion prior is parameterized as a temporal U-Net with FiLM conditioning~\cite{perez2018film}, jointly conditioned on the current observation and a sinusoidal embedding of the diffusion step, and trained under the standard $\epsilon$-prediction parameterization with $K=20$ denoising steps. The dynamics model is an ensemble of $N=5$ MLPs with independently initialized parameters that predicts the residual $s_{t+1} - s_t$ rather than the next state directly similar to~\citep{chua2018deep, janner2019trust}. The reward model is a single MLP with two heads that jointly predicts the per-step reward and a termination logit. The critic follows a clipped double-V design with value heads and an exponential moving average target network updated via Polyak averaging at rate $\tau=0.005$ every two gradient steps~\citep{fujimoto2018addressing}. All MLPs and the diffusion prior network use Swish activation~\citep{swish}. The observations are standardized by a running mean and variance maintained online over rollout data.

\paragraph{Dynamics loss.}
We train the ensemble of dynamics models $\{f_{\phi_i}\}_{i=1}^N$ to predict next states from state-action pairs. To encourage diversity across the ensemble, at each training step we draw an independent bootstrap sample (with replacement) from the current minibatch for each ensemble member, and minimize the one-step prediction error:
\begin{equation}
\mathcal{L}_{\text{dyn}}(\phi_i) 
= \mathbb{E}_{(s_t, a_t, s_{t+1}) \sim \mathcal{B}_i}
\left[ \left\| s_{t+1} - f_{\phi_i}(s_t, a_t) \right\|^2 \right],
\end{equation}
where $\mathcal{B}_i$ is a bootstrap resample of the current minibatch $\mathcal{B} \subset \mathcal{D}$.

\paragraph{Reward loss.}
The reward model $r_\xi$ is trained to approximate the environment reward from observed transitions $(s_t, a_t, r_t)$ via mean-squared regression:
\begin{equation}
\mathcal{L}_{\text{rew}}(\xi)
= \mathbb{E}_{(s_t, a_t, r_t) \sim \mathcal{D}}
\left[ \big( r_t - r_\xi(s_t, a_t) \big)^2 \right].
\end{equation}

\paragraph{Termination loss.}
The termination head $d_\xi$ shares its trunk with the reward model and outputs a logit for the probability that the trajectory terminates at step $t$. It is trained with binary cross-entropy against the observed termination signal:
\begin{equation}
\mathcal{L}_{\text{term}}(\xi)
= \mathbb{E}_{(s_t, a_t, d_t) \sim \mathcal{D}}
\left[ \mathrm{BCE}\big(d_\xi(s_t, a_t),\, d_t\big) \right],
\end{equation}
where $d_t \in \{0, 1\}$ indicates whether the environment terminated at step $t$ and $\mathrm{BCE}$ denotes binary cross-entropy applied to the predicted logit. The reward and termination heads are trained jointly under $\mathcal{L}_{\text{rew}} + \mathcal{L}_{\text{term}}$. Given predicted termination probabilities $\{\hat{d}_t\}_{t=0}^{H-1}$, we compute survival weights via the cumulative product $\prod_{k=0}^{t-1} (1 - \hat{d}_k)$, which modulate reward accumulation and bootstrap the terminal value, ensuring value estimates correctly account for early termination under model rollouts.

\paragraph{Temperature loss.}
The temperature $\lambda$ in the prior loss
(Eq.~\ref{eq:prior_loss}) controls how peaked the planner-derived
weights $w^{m} \propto \exp(J_H^{m} / \lambda)$ are over the
elite trajectories. Setting $\lambda$ by hand is particularly
tricky in our setting because the appropriate value drifts
throughout training: as the prior sharpens, beam-search returns
cluster more tightly, so a $\lambda$ that produced informative
weights early in training quickly becomes too large and flattens
the loss into a near-uniform mixture, while continually decreasing
$\lambda$ by hand requires careful scheduling per task. We
therefore adapt $\lambda$ online to target a fixed effective sample
size~\citep{ess-ref} (ESS) over the elite weights. For a batch of states
$\{s_t\}$ each with elite scores $\{J_H^{m}(s_t)\}_{m=1}^M$, the
ESS at temperature $\lambda$ is
\begin{equation}
\mathrm{ESS}(\lambda; s_t) \;=\; \frac{1}{M \sum_{m=1}^{M} (w^{m}(s_t; \lambda))^2}, \qquad
w^{m}(s_t; \lambda) = \frac{\exp(J_H^{m}(s_t) / \lambda)}{\sum_{m^\prime=1}^{M} \exp(J_H^{m^\prime}(s_t) / \lambda)},
\end{equation}
which lies in $(1/M,\, 1]$ and equals $1$ when the weights are
uniform. We optimize $\lambda$ to match a target ESS
$\rho \in (1/M, 1]$ by minimizing
\begin{equation}
\mathcal{L}_{\lambda}(\lambda) \;=\; \mathbb{E}_{s_t \sim \mathcal{D}}
\left[ \big( \mathrm{ESS}(\lambda; s_t) - \rho \big)^2 \right].
\end{equation}
This adapts the sharpness of the prior's training signal to the
current spread of beam-search returns automatically, removing the
need to schedule $\lambda$ by hand as the prior sharpens during
training.

\paragraph{Horizon curriculum.} During training, we linearly increase the planning horizon from $1$ up to its final value $H$ ($300{,}000$ environment steps for MuJoCo and $350{,}000$ environment steps for manipulation). This curriculum lets the dynamics, reward, and value models stabilize on short-horizon predictions before being asked to support long-horizon planning, where compounding model errors would otherwise dominate the objective and destabilize policy learning early in training.

\begin{table}[h]
\centering
\caption{Hyperparameters for GTP across task suites.}
\label{tab:hparams}
\resizebox{0.6\linewidth}{!}{
\begin{tabular}{lcc}
\toprule
\textbf{Hyperparameter} & \textbf{MuJoCo} & \textbf{Manipulation} \\
\midrule
\multicolumn{3}{l}{\textit{Diffusion Policy (UNet)}} \\
\quad Planning horizon $H$              & 4              & 8 \\
\quad Base channel dim                  & 64             & 32 \\
\quad Channel multipliers               & $[1, 2]$       & $[1, 2, 4]$ \\
\quad Conditioning embedding dim        & 64             & 32 \\
\midrule
\multicolumn{3}{l}{\textit{Dynamics Ensemble}} \\
\quad Hidden dims                       & \multicolumn{2}{c}{$[512, 512, 512, 512]$} \\
\quad Ensemble size                     & \multicolumn{2}{c}{5} \\
\midrule
\multicolumn{3}{l}{\textit{Reward Model}} \\
\quad Hidden dims                       & \multicolumn{2}{c}{$[256, 256, 256]$} \\
\midrule
\multicolumn{3}{l}{\textit{Value function}} \\
\quad Hidden dims                       & \multicolumn{2}{c}{$[256, 256, 256]$} \\
\midrule
Diffusion steps $K$                     & \multicolumn{2}{c}{20} \\
\midrule
\multicolumn{3}{l}{\textit{Beam-Search Sampling}} \\
\quad Total samples                 & \multicolumn{2}{c}{32} \\
\quad Beam width $B$                    & \multicolumn{2}{c}{2} \\
\quad Beam branching factor $D$                   & \multicolumn{2}{c}{4} \\
\quad Number of beams                  & \multicolumn{2}{c}{4} \\
\quad Resampling denoising steps $K_\mathrm{ref}$        & \multicolumn{2}{c}{5} \\
\midrule
\multicolumn{3}{l}{\textit{Policy Update}} \\
Number of elites                              & \multicolumn{2}{c}{4} \\
Target ESS                              & \multicolumn{2}{c}{0.8} \\
\midrule
\multicolumn{3}{l}{\textit{Optimization}} \\
\quad Learning rate                     & \multicolumn{2}{c}{$3 \times 10^{-4}$} \\
\quad Target network update $\tau$      & \multicolumn{2}{c}{0.005} \\
\quad Discount factor $\gamma$          & \multicolumn{2}{c}{0.99} \\
\bottomrule
\end{tabular}}
\end{table}

\subsection{Planning and Sampling}
At deployment, we sample candidate action sequences via beam search over the diffusion prior with a total sampling budget of $32$ trajectories. The default configuration uses $4$ parallel beams with width $2$, and apply $K_{\mathrm{ref}}=5$ denoising steps per beam expansion. Across the dynamics ensemble, we aggregate predicted returns using the median for updates and the maximum for rollout selection in the non-terminating locomotion environment \texttt{HalfCheetah-v4} and in manipulation tasks. For terminating environments, we instead use a more pessimistic aggregation scheme (minimum for updates and median for rollout), as detailed in Appendix~\ref{sec:aggregation_failure}. Furthermore, for \texttt{HalfCheetah-v4}, we adopt a narrower configuration with a single beam and $B=2$. Crucially, we find that performance is largely robust to beam search hyperparameters, as demonstrated in Appendix~\ref{app:beam_ablation}.

\section{Augmenting Planned Trajectories via Langevin Dynamics}
\label{app:langevin}

\paragraph{Motivation.}
Training the diffusion prior $p_\theta$ follows a sample-and-imitate loop: at each update step, we draw candidate action sequences from the current model, score them under the planning objective $J_H$, and regress $p_\theta$ onto the elite set reweighted by their scores. This loop has a well-known failure mode. As $p_\theta$ sharpens, its samples concentrate around a small region of trajectory space; the elite set drawn from these samples then differs little from what the model already produces, the regression target collapses onto the current mode, and learning stalls. Avoiding collapse requires a mechanism that periodically injects \emph{novel} candidates into the elite set, i.e. trajectories the prior is unlikely to have generated on its own but that nonetheless score well under $J_H$.

A common remedy when training diffusion policies is to perturb sampled actions with isotropic Gaussian noise~\citep{psenka2023learning, ma2025efficient, ding2024diffusion}. We find this unsatisfying for two reasons. First, it is unprincipled: the perturbation direction is uninformed by the objective, so most of the injected variance lands on low-value trajectories and are discarded, wasting sample budget. Second, our setting amplifies this weakness. We are learning a generator over \emph{action sequences} rather than single actions, so independent Gaussian noise across timesteps produces temporally incoherent perturbations that bear little resemblance to plausible trajectories. Worse still, we explicitly want $p_\theta$ to retain enough stochasticity to support beam search at deployment: a near-deterministic prior would defeat the purpose of using a sampling-based planner, so the perturbation mechanism should not encourage the prior to memorize a single mode.

We therefore augment the sampled trajectories with a few Langevin steps~\citep{roberts1996exponential} obtained by ascending $J_H$. Because the perturbation follows the gradient of the objective itself, novel candidates are drawn toward regions of trajectory space that are simultaneously \emph{distinct} from the current prior and \emph{high-scoring} under the model which are the two properties we desire.

\paragraph{Procedure.}
After beam search returns action sequences $\{\mathbf{a}^{b}\}_{b=1}^{B}$, we generate Langevin counterparts $\{\tilde{\mathbf{a}}^{b}\}_{b=1}^{B}$ by running $L_\text{lan}=4$ steps of
\begin{equation}
    \mathbf{a} \;\leftarrow\;\mathbf{a} \;+\; \eta\,\nabla_{\mathbf{a}}\, J_H(s,\mathbf{a}) \;+\; \sqrt{2\eta}\,\boldsymbol{\epsilon}, \qquad \boldsymbol{\epsilon}\sim\mathcal{N}(0, I),
    \label{eq:langevin}
 \end{equation}
with the gradient taken with respect to the full action sequence. The step size $\eta$ and noise scale $\sigma$ are coupled via $\eta = \sigma^2/2$, recovering the standard Langevin interpretation.  To let the implicit target distribution sharpen as the planning horizon extends and the model becomes more reliable, we schedule $\eta$ on a cosine curve, taking the standard deviation of noise from $0.2$ down to $0.1$ over the same window as the horizon curriculum. The combined set $\{\mathbf{a}^{b}\}\cup\{\tilde{\mathbf{a}}^{b}\}$ is then re-scored under $J_H$, and the policy loss is computed over the elite set obtained from the augmented pool.

\section{Aggregation Sensitivity and Failure Modes}
\label{sec:aggregation_failure}

In this section, we study the impact of ensemble aggregation in
termination-sensitive environments and show that a more conservative,
risk-sensitive scheme substantially improves performance. We observe that the
choice of aggregation across dynamics models plays a critical role in
environments with early termination, such as \texttt{Ant-v4}. In these
settings, our default configuration (median aggregation during training
updates and max aggregation during rollout evaluation) leads to degraded
performance. This effect arises from the interaction between model bias and
termination sensitivity: small prediction errors compound over the planning
horizon and quickly produce unstable states that trigger termination. Under
max aggregation, the planner favors trajectories that appear high-return under
at least one ensemble member, even when they are not consistently supported
across the ensemble, yielding brittle plans that fail under slight model
mismatch.

To address this, we adopt a more pessimistic aggregation scheme, using
minimum aggregation during updates and median aggregation during rollout
evaluation. This reduces sensitivity to overestimated returns from individual
models and biases planning toward trajectories that are consistently predicted
to perform well across the ensemble, leading to significantly improved
stability and return. Figure~\ref{fig:fail-ablation} reports performance on
\texttt{Ant-v4} under both schemes, following the experimental procedure of
Section~\ref{sec:mujoco}. The conservative aggregation yields a substantial
improvement, recovering strong performance on this domain.

More broadly, aggregation defines an implicit risk preference over the
ensemble-induced uncertainty: optimistic choices emphasize best-case
predictions, while conservative choices approximate a worst-case objective.
This connects naturally to robust MPC~\cite{bemporad2007robust}, which selects actions that perform reliably across an admissible set of models rather than optimally under a
nominal one; our minimum aggregation strategy can be viewed as a sample-based
surrogate for such a worst-case objective over the ensemble. It also relates
to risk-sensitive and distributional control, where the planning objective
is shaped by a functional of the return distribution rather than its
expectation, and our use of the aggregation strategies corresponds
to a simple quantile-based instance of this idea. We leave a more thorough
investigation of these connections to future work, including richer
risk-sensitive objectives~\citep{chow2015risk}, and
UCB-style criteria~\citep{auer2002finite} that trade off mean return against ensemble disagreement.

\begin{figure}[t]
    \centering
    \includegraphics[width=0.65\linewidth]{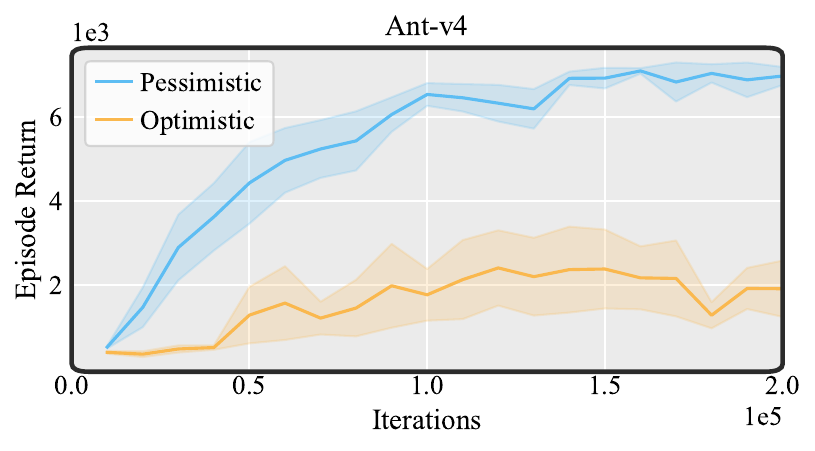}
    \vspace{-10pt}
    \caption{\textbf{Effect of ensemble aggregation on Ant-v4.} Pessimistic aggregation (min/median for update/rollout) substantially outperforms optimistic aggregation (median/max) on this termination-sensitive environment.}
    \label{fig:fail-ablation}
    \vspace{-10pt}
\end{figure}

\vspace{-5pt}
\section{Beam Search Parameters}\label{app:beam_ablation}
\vspace{-5pt}

We sweep the two beam search hyperparameters (the number of beams and the per-beam width $B$) under a fixed sampling budget of 32 on \texttt{HalfCheetah-v4} and report results in Table~\ref{tab:beam_ablation}. Performance is largely insensitive to either choice within the swept range: seven of the eight configurations exceed the strongest baseline reported in Section~\ref{sec:experiments}, with the remaining configuration still within range. This suggests our method is robust to beam search hyperparameter selection, and the gains we report are not the product of careful tuning along this axis.

\vspace{-10pt}
\begin{table}[h]
\centering
\caption{\texttt{HalfCheetah-v4} environment best mean returns and standard errors over 200k iterations with 5 random seeds for different beam widths $B$ and number of beams.}
\label{tab:beam_ablation}
\resizebox{0.8\linewidth}{!}{
\begin{tabular}{c|cccc}
\toprule
 beam width $ \backslash$ num. beams  & 1 & 2 & 4 & 8 \\
\midrule
2 & \textbf{13944} $\pm$ \textbf{498} & 13830 $\pm$ 517 & 13338 $\pm$ 174 & 12892 $\pm$ 778 \\
4 & 11819 $\pm$ 481 & 13587 $\pm$ 572 & 12614 $\pm$ 792 & 12268 $\pm$ 326 \\
\bottomrule
\end{tabular}}
\vspace{-10pt}
\end{table}

\section{PushT-Vision: Environment Details and Extended Training}
\label{app:pusht}

\paragraph{Environment.}
PushT is a planar manipulation task commonly used in imitation learning~\citep{florence2022implicit, chi2023diffusion}, in which a circular end-effector pushes a T-shaped block from a randomized initial pose into a fixed target pose. We adapt the publicly available
implementation to the online setting; the resulting code is released
alongside our codebase\footnote{\url{\codeurl}}. Observations are
$96\times96$ RGB renderings (cf. Figure~\ref{fig:pusht-3M}) of the workspace, actions are velocity commands, and the reward is a shaped dense signal combining squared
pose error between the block and target, the agent--block distance,
and small penalties for wall contacts. Episodes
last at most 300 steps and are deemed successful when the
intersection-over-union between the block and target exceeds $0.95$.

\paragraph{Extended training.}
The 1M-step budget used in Section~\ref{sec:manipulation} is short relative to standard
visuomotor learning protocols, which raises the question of whether the
0\% success rates we report for the model-free and diffusion-policy
baselines reflect a fundamental limitation or simply insufficient
training. To address this, we re-run all baselines and GTP on
PushT-Vision for 3M environment interactions under otherwise identical
settings (5 seeds, 4 parallel environments, 100 evaluation trials per
seed). Results are shown in Table~\ref{tab:pusht-3M} and
Figure~\ref{fig:pusht-3M}.

% --- Replace the placeholders below with your actual numbers ---
\begin{table}[t]
\centering
\caption{PushT-Vision success rate (\%) at 3M environment interactions.
Mean and standard error over 5 seeds.}
\label{tab:pusht-3M}
\resizebox{0.8\linewidth}{!}{
\begin{tabular}{lccccccc}
\toprule
 & SAC & TD3 & QSM & DPMD & DreamerV3 & TD-MPC2 & GTP (ours) \\
\midrule
Success Rate & 0 $\pm$ 0 & 0 $\pm$ 0 & 0 $\pm$ 0 & 0 $\pm$ 0 & 0 $\pm$ 0 & 97 $\pm$ 1 & \textbf{99 $\pm$ 1} \\
\bottomrule
\end{tabular}}
\vspace{-10pt}
\end{table}

\begin{figure}[t]
    \centering
    \includegraphics[width=0.9\linewidth]{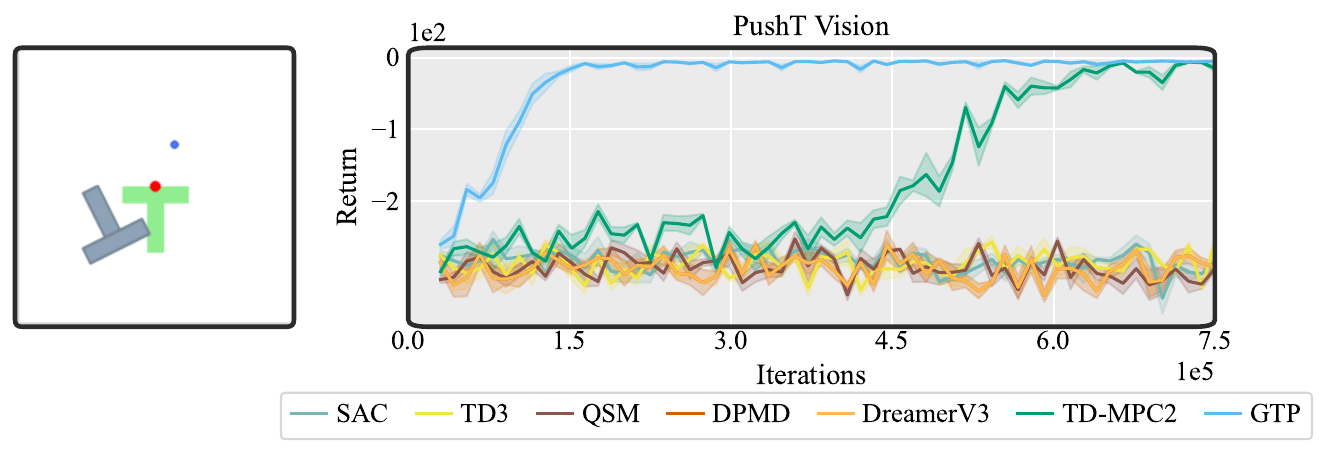}
    \caption{\textbf{Left:} Rendering of the PushT-Vision environment.
    \textbf{Right:} Evaluation returns on PushT-Vision over 3M environment interactions. Solid lines show the mean and shaded regions the standard error over 5 seeds.}
    \label{fig:pusht-3M}
\vspace{-20pt}
\end{figure}

% Brief discussion paragraph --- adjust based on what the numbers show.
\paragraph{Discussion.}
The results sharpen the picture from Section~\ref{sec:manipulation}. Even with $3\times$ more
environment interactions, SAC, TD3, QSM, DPMD, and DreamerV3 remain at
$0\%$ success on PushT-Vision, indicating that their failure at 1M is
structural rather than a matter of training budget: reactive policies
updated against single-step value targets do not discover the coherent,
multi-step push sequences this task requires, and additional samples
alone do not change this. TD-MPC2 is the one baseline that eventually
solves the task, reaching a statistical tie with
GTP, but only after substantially more interactions
(see Fig.~\ref{fig:pusht-3M}). We attribute this gap to the role of the
diffusion prior: GTP samples coherent action sequences from a learned
trajectory distribution from the outset, whereas TD-MPC2 must first
bootstrap a sufficiently accurate value function before its
short-horizon planner can recover similar behaviors. Taken together,
these extended runs support our claim that GTP scales to
high-dimensional observations, and additionally show that on
PushT-Vision the diffusion prior translates into a clear sample-efficiency advantage over the strongest baseline rather than a final-performance one.
\vspace{-5pt}
\section{Baseline Implementations}\label{app:baselines}
\vspace{-5pt}

We describe the source and any modifications for each baseline.
Across all methods, we use the original authors' recommended
hyperparameters unless noted otherwise, and run every method under
the shared evaluation protocol described in
Sec.~\ref{sec:experiments}.

\paragraph{SAC and TD3.}
We implement SAC \citep{haarnoja2018soft} and
TD3 \citep{fujimoto2018addressing} ourselves in JAX, following the
original papers.

\paragraph{DPMD.}
We use the original authors' codebase~\citep{ma2025efficient} and
their published hyperparameters without modification.

\paragraph{QSM.}
We use the implementation from \citet{ma2025efficient}\footnote{\url{https://github.com/mahaitongdae/diffusion_policy_online_rl}},
which reproduces the original method~\citep{psenka2023learning}, with the
hyperparameters provided in that repository.

\paragraph{TD-MPC2.}
The original TD-MPC2 codebase~\citep{hansen2023td} was
difficult to adapt to our training and evaluation protocol, so we
reimplemented the method in JAX. We follow the original
architecture, training objectives, and hyperparameters.

\paragraph{DreamerV3.}
We use the original authors' codebase~\citep{hafner2023mastering}
with default hyperparameters, modified only to interface with our
evaluation harness.

\paragraph{CEM-MPC and CEM-MPC-V.}
Both variants use a custom CEM solver that we tune to be a faithful
sampling-based planner. CEM-MPC plans without a terminal value and
uses a planning horizon of $H=20$, long enough to capture
task-relevant returns under the learned dynamics. CEM-MPC-V adds a
learned terminal value and uses the same horizon as GTP for a
direct comparison with the diffusion-prior planner. Both variants
share GTP's learned dynamics model, and CEM-MPC-V additionally
shares GTP's value function; this isolates the effect of the
action prior. We tune the number of samples, elite fraction, number
of CEM iterations, and noise schedule per task family
(locomotion vs.\ manipulation), and report the best result over the
last three checkpoints of each run.

\vspace{-10pt}
\section{Handling Image Observations}\label{app:vision}
\vspace{-10pt}

\paragraph{GTP.}
To handle image observations, GTP encodes the input frame into a
compact latent and operates on this latent in place of the
proprioceptive state. The autoencoder is trained jointly with the
dynamics and reward models, with a reconstruction loss added to the
existing dynamics and reward losses; this lets representation
learning be shaped by the predictive objectives rather than
reconstruction alone.

\paragraph{TD-MPC2 and DreamerV3.}
We use their original architectures, as both of these methods are already designed to handle both state and image inputs.

\paragraph{Other baselines.}
We use the same encoder--decoder architecture as GTP, but train it
with a standalone reconstruction loss on the replay buffer since
these methods do not learn dynamics or reward models. The resulting
latent is used in place of the state input.

\paragraph{Encoder--decoder architecture.}
The encoder is a 4-block convolutional network (channels
$32, 32, 64, 64$; strides $2, 2, 2, 1$; GroupNorm and swish after
each block) followed by a dense projection to a latent of size $d$
with a final LayerNorm. The decoder projects the latent to a
$12{\times}12$ feature map and applies bilinear upsampling followed
by a conv block at each stage (channel schedule $64, 32, 32$),
with a final conv mapping to RGB.

\vspace{-10pt}
\section{Computational Cost}\label{app:timing}
\vspace{-10pt}

We report wall-clock timing for action selection and gradient updates
to characterize the computational profile of GTP relative to
representative baselines. Two costs are relevant in online decision
making: the per-step inference cost incurred during environment
interaction and the per-update cost incurred during training.
GTP introduces additional computation through diffusion-based
trajectory generation and beam search, which we quantify below.
We subsequently investigate strategies for reducing these costs
in Appendices~\ref{app:bootstrap} and~\ref{app:efficient_inference}.

\paragraph{Setup.}
All measurements were taken on a single NVIDIA RTX 5090 GPU
(32\,GB), using our JAX implementation. Times are averaged
over 1{,}000 steps after a warm-up phase to exclude compilation
overhead. Each baseline was run under a matched batch size
of 256 and identical hardware.

\paragraph{Action selection.}
At comparable planning horizons, GTP achieves action selection
times similar to TD-MPC2: GTP ($H=8$) requires 34.25\,ms per
action compared to 36.20\,ms for TD-MPC2 at the same horizon.
Despite the additional diffusion denoising and beam search
operations, GTP maintains competitive inference costs by
leveraging its generative prior to propose promising trajectories,
requiring substantially fewer candidates than Gaussian
sampling-based planners.

\paragraph{Training updates.}
The primary computational overhead of GTP arises from
generating elite trajectories through beam search at every
gradient update. In the original implementation, GTP ($H=8$)
requires 157.15\,ms per update, compared to 3.99\,ms for
TD-MPC2 at the same horizon. This additional cost reflects
the optimization of a multi-step diffusion prior using
planner-generated targets rather than conventional
single-step bootstrapping. While this increases training
computation, GTP achieves strong sample efficiency across
our benchmarks, making the trade-off particularly relevant
when environment interactions are expensive.

\vspace{-10pt}
\begin{table}[h]
\centering
\caption{
\textbf{Computational cost of GTP and representative baselines.}
Per-step action selection and training update times (ms)
under matched hardware and batch settings.
}
\label{tab:timing}
\resizebox{0.5\linewidth}{!}{
\begin{tabular}{lcc}
\toprule
Method & Action selection (ms) & Update (ms) \\
\midrule
SAC & 0.06 & 0.78 \\
DPMD & 0.73 & 2.18 \\
TD-MPC2 ($H=3$) & 17.19 & 2.44 \\
TD-MPC2 ($H=8$) & 36.20 & 3.99 \\
\midrule
GTP ($H=1$) & 1.02 & 8.54 \\
GTP ($H=4$) & 21.20 & 82.83 \\
GTP ($H=8$) & 34.25 & 157.15 \\
\bottomrule
\end{tabular}}
\end{table}

\begin{figure}[t]
    \centering
    \includegraphics[width=0.9\linewidth]{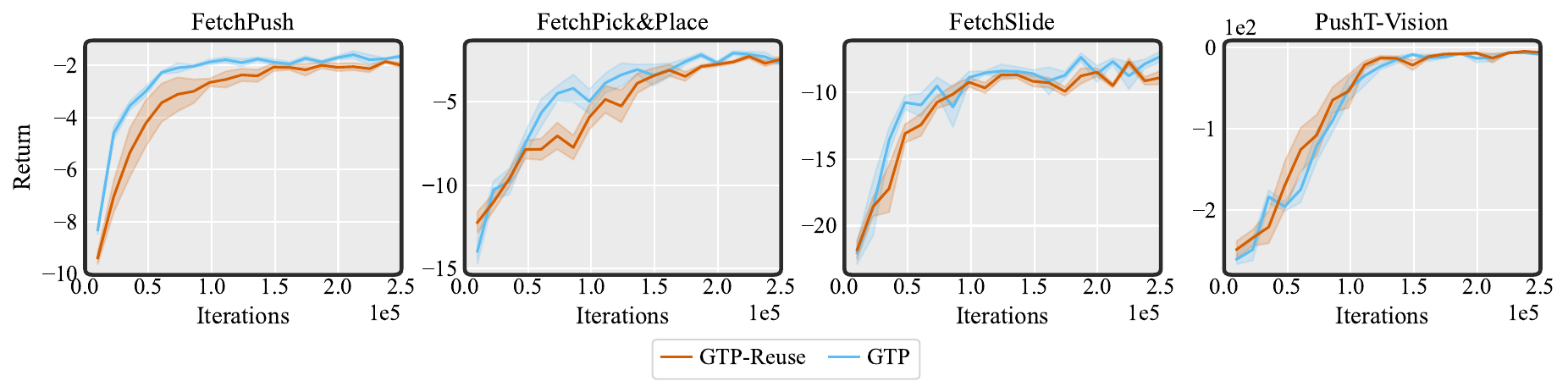}

    \vspace{0.5em}

    \resizebox{0.7\linewidth}{!}{
    \begin{tabular}{l|cccc}
        \toprule
        Algorithm & FetchPush & FetchPick\&Place
                  & FetchSlide & PushT-Vision \\
        \midrule
        GTP-Reuse & \textbf{100 $\pm$ 0} & \textbf{100 $\pm$ 0}
                     & 69 $\pm$ 1 & 98 $\pm$ 2 \\
        GTP & 100 $\pm$ 0 & 98 $\pm$ 1
               & \textbf{76 $\pm$ 3} & \textbf{99 $\pm$ 1} \\
        \bottomrule
    \end{tabular}}

    \caption{
        \textbf{Trajectory reuse reduces training overhead while
        preserving performance.}
        Reusing elite trajectories  maintains comparable learning curves (top) and final success rates (bottom).
    }
    \label{fig:bootstrap}
    \vspace{-10pt}
\end{figure}

\vspace{-5pt}
\section{Efficient Training via Trajectory Reuse}\label{app:bootstrap}
\vspace{-5pt}

A significant computational overhead in GTP arises from generating elite trajectories through beam search at every gradient update (cf. Alg.~\ref{alg:gtp_training}). However, generating fresh elites at every iteration is not necessary, as previously optimized trajectories can be reused across multiple updates.

To reduce this overhead, we introduce a trajectory reuse
variant that caches planner-generated elite trajectories and
reuses them for subsequent updates, following a similar approach
to~\citet{wang2025bootstrapped}. Specifically, we refresh the
cached elites once every 10 gradient updates instead of
regenerating them at every iteration. This reduces the frequency
of elite trajectory generation by a factor of 10 while preserving the same training objective.

\paragraph{Learning performance.}
Figure~\ref{fig:bootstrap} compares the learning curves and
final success rates of the original online training procedure
and its trajectory reuse variant. Trajectory reuse achieves
comparable performance across all four manipulation tasks,
suggesting that frequent regeneration of elite trajectories
is not essential for effective learning.

\paragraph{Computational savings.}
Trajectory reuse reduces the amortized update time of GTP
($H=8$) from 157.15\,ms to 19.12\,ms, corresponding to
an $8.2\times$ speedup. These results demonstrate that reusing
planner-generated trajectories substantially reduces
the computational cost of GTP without compromising its
learning performance.

\vspace{-5pt}
\section{Efficient Inference}\label{app:efficient_inference}
\vspace{-5pt}

Although GTP uses iterative beam search during training,
the learned generative prior can support efficient action
selection at deployment. We investigate two approaches
for reducing inference-time computation: removing explicit
planning and executing multiple actions between replanning
steps.

\paragraph{Planning-free inference.}
Since GTP trains its generative prior on planner-discovered
trajectories, the learned prior can generate high-quality
actions without explicit trajectory optimization. To evaluate
this property, we vary the number of candidate trajectories
at inference on FetchPick\&Place, including a single-candidate
variant without beam search.

As shown in Table~\ref{tab:inference_timing}, removing
inference-time planning reduces action selection time from
34.25\,ms to 8.61\,ms, a $4.0\times$ speedup, while
maintaining comparable success rates.
This demonstrates that planning during training can be
amortized into the generative prior, enabling efficient
deployment without repeated trajectory optimization.

\begin{table}[t]
\centering

\begin{minipage}[t]{0.56\linewidth}
\centering
\captionof{table}{
\textbf{GTP maintains performance with reduced inference-time planning.}
Action selection time and success rate on FetchPick\&Place
for different sampling budgets.
}
\label{tab:inference_timing}
\vspace{0.5em}
\resizebox{\linewidth}{!}{
\begin{tabular}{lcc}
\toprule
Samples & Action selection (ms) & Success rate (\%) \\
\midrule
32 & 34.25 & 98 \\
16 & 32.20 & 99 \\
8 & 32.26 & 99 \\
1 (no planning) & \textbf{8.61} & 99 \\
\bottomrule
\end{tabular}}
\end{minipage}
\hfill
\begin{minipage}[t]{0.41\linewidth}
\centering
\captionof{table}{
\textbf{GTP supports multi-step execution.}
Success rates (\%) on FetchPick\&Place when executing
multiple actions between replanning steps.
}
\label{tab:replanning}
\vspace{0.5em}
\resizebox{\linewidth}{!}{
\begin{tabular}{lcccc}
\toprule
Replanning interval & 1 & 2 & 4 & 8 \\
\midrule
Success rate (\%) & 98 & 98 & 96 & 92 \\
\bottomrule
\end{tabular}}
\end{minipage}

\end{table}

\paragraph{Multi-step execution.}
In our default implementation, GTP follows receding-horizon
MPC and replans after every environment step. However,
the generated action sequences can also be executed for
multiple steps before replanning, reducing the frequency
of expensive planning operations.

We evaluate this strategy on FetchPick\&Place using a
model trained with planning horizon $H=8$. As shown in
Table~\ref{tab:replanning}, executing multiple actions
before replanning preserves strong performance.

These results suggest that GTP's generated trajectories
remain effective even when executed partially open-loop.
While less frequent replanning reduces computational
overhead, replanning at every step remains preferable
during online training, when inaccurate dynamics and
rapidly changing policies can amplify prediction errors.
At deployment, however, multi-step execution provides
an additional way to reduce planning frequency with
limited performance degradation.

\section{Convergence Analysis of the Generative Prior}
\label{app:theory}
In this section, we provide a theoretical characterization of the generative prior update introduced in Section ~\ref{sec:theory}.
Building on the connection between reweighted diffusion training
and policy improvement~\citep{ma2025efficient}, we show that
the idealized GTP update corresponds to a KL-regularized
improvement step in trajectory space. We then prove
Theorem~\ref{thm:prior_convergence}, establishing monotonic
improvement and asymptotic convergence, and discuss the
implications for GTP's practical components.
\subsection{Idealized Prior Update}
Consider a fixed state $s$ and let $\mathcal{A}^H$ denote the
space of action sequences. We suppress the dependence on $s$
and write $J(\mathbf{a})=J_H(s,\mathbf{a})$. We assume that
$J$ is bounded and fixed throughout the analysis, corresponding
to holding the learned dynamics, reward, and value models fixed.
Let $p_n$ denote the generative prior after $n$ updates.
In GTP, candidate trajectories are evaluated under $J$
and assigned weights proportional to
$\exp(J(\mathbf{a})/\lambda)$, where $\lambda>0$ is
the temperature parameter.
Consider the idealized setting where infinitely many
trajectories are sampled from $p_n$, and the diffusion
model has sufficient capacity and is optimized exactly.
Under these conditions, the weighted denoising objective
recovers the exponentially tilted distribution
\begin{equation}
    p_{n+1}(\mathbf{a})
    =
    \frac{
        p_n(\mathbf{a})
        \exp(J(\mathbf{a})/\lambda)
    }{
        \mathbb{E}_{p_n}[
            \exp(J(\mathbf{a})/\lambda)
        ]
    }.
    \label{eq:theory_update}
\end{equation}
This follows from the reweighted score-matching interpretation
of diffusion policy optimization~\citep{ma2025efficient},
extended from individual actions to action sequences.
The update admits the variational interpretation
\begin{equation}
    p_{n+1}
    =
    \arg\max_{q\ll p_n}
    \left\{
        \mathbb{E}_{q}[J]
        -
        \lambda D_{\mathrm{KL}}(q\|p_n)
    \right\},
    \label{eq:theory_variational}
\end{equation}
where $q\ll p_n$ denotes absolute continuity.
This is a KL-regularized improvement step, closely related
to policy mirror descent~\citep{tomar2020mirror}, applied
to trajectory distributions rather than single step policies.

\subsection{Convergence Analysis}
We now prove Theorem~\ref{thm:prior_convergence}.
The result follows from the variational characterization
of exponential tilting and its concentration properties.
Throughout, $J$ is bounded and $\lambda>0$ is fixed. Let
$J^\star=\sup_{\mathbf{a}}J(\mathbf{a})$ and, for $\delta>0$,
define the set of $\delta$-optimal trajectories
\begin{equation}
    U_\delta
    =
    \left\{
        \mathbf{a}:
        J(\mathbf{a})\geq J^\star-\delta
    \right\}.
\end{equation}

\begin{proof}[Proof of Theorem~\ref{thm:prior_convergence}]
\textit{Part (i): Monotonic improvement.}
By Eq.~\eqref{eq:theory_variational}, $p_{n+1}$ maximizes
$\mathbb{E}_{q}[J]-\lambda D_{\mathrm{KL}}(q\|p_n)$ over
$q\ll p_n$. Since $q=p_n$ is feasible and
$D_{\mathrm{KL}}(p_n\|p_n)=0$,
\begin{equation}
    \mathbb{E}_{p_{n+1}}[J]
    -
    \lambda D_{\mathrm{KL}}(p_{n+1}\|p_n)
    \geq
    \mathbb{E}_{p_n}[J].
\end{equation}
Rearranging and using $D_{\mathrm{KL}}\geq 0$,
\begin{equation}
    \mathbb{E}_{p_{n+1}}[J]
    \geq
    \mathbb{E}_{p_n}[J]
    +
    \lambda D_{\mathrm{KL}}(p_{n+1}\|p_n)
    \geq
    \mathbb{E}_{p_n}[J].
\end{equation}

\textit{Part (ii): Convergence.}
Assume $p_0(U_\delta)>0$ for every $\delta>0$.

\emph{Step 1 (closed form).} We claim that for every $n\geq 0$,
\begin{equation}
    p_n(\mathbf{a})
    =
    \frac{
        p_0(\mathbf{a})
        \exp(nJ(\mathbf{a})/\lambda)
    }{
        \mathbb{E}_{p_0}[
            \exp(nJ/\lambda)
        ]
    }.
    \label{eq:theory_closed_form}
\end{equation}
The case $n=0$ is immediate. If it holds for $n$, then by
Eq.~\eqref{eq:theory_update},
$p_{n+1}\propto p_n e^{J/\lambda}\propto p_0 e^{(n+1)J/\lambda}$,
and the normalizer is finite and positive because $J$ is bounded.

\emph{Step 2 (optimality gap).} Fix $n\geq1$ and $\delta>0$.
Eq.~\eqref{eq:theory_closed_form} is the exponential tilt of
$p_0$ at temperature $\lambda/n$, so the variational
characterization~\eqref{eq:theory_variational}, applied with
$p_0$ in place of $p_n$ and $\lambda/n$ in place of $\lambda$,
shows that $p_n$ maximizes
\begin{equation}
    \mathcal{L}_n(q)
    =
    \mathbb{E}_{q}[J]
    -
    \frac{\lambda}{n}
    D_{\mathrm{KL}}(q\|p_0)
\end{equation}
over $q\ll p_0$. Let $q_\delta=p_0(\cdot\mid U_\delta)$, which is
well defined since $p_0(U_\delta)>0$. Its density with respect to
$p_0$ is $\mathbf{1}_{U_\delta}/p_0(U_\delta)$, and $J\geq J^\star-\delta$
on $U_\delta$, so
\begin{equation}
    \mathbb{E}_{q_\delta}[J]
    \geq J^\star-\delta,
    \qquad
    D_{\mathrm{KL}}(q_\delta\|p_0)
    =
    \log\frac{1}{p_0(U_\delta)}.
\end{equation}
Combining $D_{\mathrm{KL}}(p_n\|p_0)\geq0$ with the optimality
of $p_n$,
\begin{equation}
    \mathbb{E}_{p_n}[J]
    \;\geq\;
    \mathcal{L}_n(p_n)
    \;\geq\;
    \mathcal{L}_n(q_\delta)
    \;\geq\;
    J^\star-\delta
    -\frac{\lambda}{n}\log\frac{1}{p_0(U_\delta)},
\end{equation}
that is,
\begin{equation}
    J^\star-\mathbb{E}_{p_n}[J]
    \leq
    \delta
    +
    \frac{\lambda}{n}
    \log\frac{1}{p_0(U_\delta)}.
    \label{eq:theory_gap}
\end{equation}
Letting $n\to\infty$ gives
$\limsup_{n\to\infty}\big(J^\star-\mathbb{E}_{p_n}[J]\big)\leq\delta$.
Since $\delta>0$ is arbitrary and $\mathbb{E}_{p_n}[J]\leq J^\star$
for all $n$,
\begin{equation}
    \lim_{n\to\infty}
    \mathbb{E}_{p_n}[J]
    =
    J^\star.
\end{equation}

\emph{Step 3 (concentration).} Fix $\delta>0$. On $U_\delta^c$ we have
$J<J^\star-\delta$, and on $U_{\delta/2}$ we have $J\geq J^\star-\delta/2$.
Bounding the numerator of Eq.~\eqref{eq:theory_closed_form} from above
on $U_\delta^c$, and the normalizer from below by restricting it to
$U_{\delta/2}$,
\begin{equation}
    p_n(U_\delta^c)
    =
    \frac{\mathbb{E}_{p_0}\big[e^{nJ/\lambda}\mathbf{1}_{U_\delta^c}\big]}
         {\mathbb{E}_{p_0}\big[e^{nJ/\lambda}\big]}
    \leq
    \frac{e^{n(J^\star-\delta)/\lambda}}
         {e^{n(J^\star-\delta/2)/\lambda}\,p_0(U_{\delta/2})}
    =
    \frac{
        \exp(-n\delta/(2\lambda))
    }{
        p_0(U_{\delta/2})
    }.
    \label{eq:theory_concentration}
\end{equation}
Since $p_0(U_{\delta/2})>0$, the right-hand side converges to zero, so
\begin{equation}
    \lim_{n\to\infty}
    p_n(U_\delta)=1,
    \qquad \forall\,\delta>0.
\end{equation}
\end{proof}

\subsection{Relation to GTP}
The preceding analysis considers an idealized
self-imitation process, whereas GTP incorporates
several additional components.

\paragraph{Beam search.}
The convergence theorem assumes that trajectories
are sampled directly from the current prior.
In practice, GTP instead generates improved
trajectory proposals through beam search.
Let $e$ denote the distribution of trajectories
produced by the planner. Under exact distillation,
the updated prior satisfies
\begin{equation}
    p^+(\mathbf{a})
    \propto
    e(\mathbf{a})
    \exp(J(\mathbf{a})/\lambda).
\end{equation}
By the same variational argument,
\begin{equation}
    \mathbb{E}_{p^+}[J]
    \geq
    \mathbb{E}_{e}[J]
    +
    \lambda D_{\mathrm{KL}}(p^+\|e).
\end{equation}
Thus, the distilled prior achieves an expected
planning objective at least as high as that
of the proposal distribution. If beam search
produces proposals satisfying
$\mathbb{E}_{e}[J]\geq\mathbb{E}_{p_n}[J]$,
the resulting update is also monotonic.
This provides an interpretation of beam search
as an improvement mechanism whose discoveries
are distilled into the generative prior.

\paragraph{Temperature adaptation.}
GTP adapts the temperature to maintain a
target effective sample size (ESS) over
planner-generated trajectories. In the
population limit, the normalized ESS is
related to the divergence between the
updated prior and its proposal distribution:
\begin{equation}
    \mathrm{ESS}(\lambda)
    =
    \frac{1}{
        1+\chi^2(p^+\|e)
    }.
\end{equation}
Consequently, targeting a fixed ESS
corresponds to controlling the size of
the prior update in terms of
$\chi^2$ divergence. This prevents
excessively concentrated importance weights
while allowing the prior to progressively
shift toward higher-return trajectories.

\paragraph{Langevin augmentation.}
A limitation of the idealized update is
that it cannot introduce probability mass
outside the support of the current prior:
\begin{equation}
    \operatorname{supp}(p_{n+1})
    \subseteq
    \operatorname{supp}(p_n).
\end{equation}
Consequently, if the prior loses support
around optimal trajectories, subsequent
updates cannot recover these regions,
violating the positivity assumption
required for convergence.
GTP addresses this issue by augmenting
planner-generated trajectories with
Langevin perturbations
(App.~\ref{app:langevin}). The stochastic
component introduces additional support,
while the objective gradient directs
perturbations toward high-scoring regions.
This helps mitigate mode collapse and
maintains opportunities to discover
improved solutions.

\paragraph{Scope and limitations.}
Our theoretical results characterize the idealized prior update under a fixed planning objective, infinitely many samples, and exact diffusion model fitting. They establish monotonic improvement and asymptotic convergence in this setting, but do not provide convergence guarantees for finite-budget beam search, imperfect function approximation, or jointly learned dynamics and value models. Nevertheless, the analysis provides a theoretical foundation for iterative trajectory improvement and prior distillation, while clarifying the roles of GTP's key components in promoting effective learning.

\newpage
\section*{NeurIPS Paper Checklist}

\begin{enumerate}

\item {\bf Claims}
    \item[] Question: Do the main claims made in the abstract and introduction accurately reflect the paper's contributions and scope?
    \item[] Answer: \answerYes{} % Replace by \answerYes{}, \answerNo{}, or \answerNA{}.
    \item[] Justification:The abstract and introduction provide an accurate and consistent summary of the paper’s contributions, including the proposed method, its formulation, and the empirical improvements demonstrated in the experiments.

\item {\bf Limitations}
    \item[] Question: Does the paper discuss the limitations of the work performed by the authors?
    \item[] Answer: \answerYes{} % Replace by \answerYes{}, \answerNo{}, or \answerNA{}.
    \item[] Justification: The paper discusses key limitations and assumptions underlying the proposed approach and outlines directions for future work to address them.

\item {\bf Theory assumptions and proofs}
    \item[] Question: For each theoretical result, does the paper provide the full set of assumptions and a complete (and correct) proof?
    \item[] Answer: \answerYes{} % Replace by \answerYes{}, \answerNo{}, or \answerNA{}.
    \item[] Justification: The paper presents a theoretical convergence result in Section~\ref{sec:theory}, with assumptions, a formal proof, and a discussion of its applicability and limitations provided in Appendix~\ref{app:theory}.

    \item {\bf Experimental result reproducibility}
    \item[] Question: Does the paper fully disclose all the information needed to reproduce the main experimental results of the paper to the extent that it affects the main claims and/or conclusions of the paper (regardless of whether the code and data are provided or not)?
    \item[] Answer: \answerYes{} % Replace by \answerYes{}, \answerNo{}, or \answerNA{}.
    \item[] Justification: The paper provides sufficient detail on the experimental setup, including model architecture, training procedures, hyperparameters, and evaluation protocols, to enable reproduction of the main results and support the paper’s claims and conclusions.

\item {\bf Open access to data and code}
    \item[] Question: Does the paper provide open access to the data and code, with sufficient instructions to faithfully reproduce the main experimental results, as described in supplemental material?
    \item[] Answer: \answerYes{}{} % Replace by \answerYes{}, \answerNo{}, or \answerNA{}.
    \item[] Justification: We provide access to the code, along with documentation and instructions to reproduce the main experimental results

\item {\bf Experimental setting/details}
    \item[] Question: Does the paper specify all the training and test details (e.g., data splits, hyperparameters, how they were chosen, type of optimizer) necessary to understand the results?
    \item[] Answer: \answerYes{} % Replace by \answerYes{}, \answerNo{}, or \answerNA{}.
    \item[] Justification: The paper specifies all relevant training and evaluation details, hyperparameters and their selection, and other implementation specifics necessary to understand and interpret the results.

\item {\bf Experiment statistical significance}
    \item[] Question: Does the paper report error bars suitably and correctly defined or other appropriate information about the statistical significance of the experiments?
    \item[] Answer: \answerYes{} % Replace by \answerYes{}, \answerNo{}, or \answerNA{}.
    \item[] Justification: All results are averaged over 5 independent trials, and we report standard errors to quantify variability and support the statistical reliability of the findings.

\item {\bf Experiments compute resources}
    \item[] Question: For each experiment, does the paper provide sufficient information on the computer resources (type of compute workers, memory, time of execution) needed to reproduce the experiments?
    \item[] Answer: \answerYes{} % Replace by \answerYes{}, \answerNo{}, or \answerNA{}.
    \item[] Justification: We provide the details in Appendix~\ref{app:timing}.
    
\item {\bf Code of ethics}
    \item[] Question: Does the research conducted in the paper conform, in every respect, with the NeurIPS Code of Ethics \url{https://neurips.cc/public/EthicsGuidelines}?
    \item[] Answer: \answerYes{} % Replace by \answerYes{}, \answerNo{}, or \answerNA{}.
    \item[] Justification: The research complies with the NeurIPS Code of Ethics.

\item {\bf Broader impacts}
    \item[] Question: Does the paper discuss both potential positive societal impacts and negative societal impacts of the work performed?
    \item[] Answer: \answerNA{} % Replace by \answerYes{}, \answerNo{}, or \answerNA{}.
    \item[] Justification: The paper focuses on methodological contributions and does not present direct societal impacts.
    
\item {\bf Safeguards}
    \item[] Question: Does the paper describe safeguards that have been put in place for responsible release of data or models that have a high risk for misuse (e.g., pre-trained language models, image generators, or scraped datasets)?
    \item[] Answer: \answerNA{} % Replace by \answerYes{}, \answerNo{}, or \answerNA{}.
    \item[] Justification: The paper does not introduce or release artifacts that pose a high risk of misuse.

\item {\bf Licenses for existing assets}
    \item[] Question: Are the creators or original owners of assets (e.g., code, data, models), used in the paper, properly credited and are the license and terms of use explicitly mentioned and properly respected?
    \item[] Answer: \answerYes{} % Replace by \answerYes{}, \answerNo{}, or \answerNA{}.
    \item[] Justification: All external assets used in the paper are properly credited, and their licenses and terms of use are explicitly acknowledged and respected.

\item {\bf New assets}
    \item[] Question: Are new assets introduced in the paper well documented and is the documentation provided alongside the assets?
    \item[] Answer: \answerYes{} % Replace by \answerYes{}, \answerNo{}, or \answerNA{}.
    \item[] Justification: Newly introduced assets are documented and accompanied by instructions to facilitate their use and reproduction of the results.

\item {\bf Crowdsourcing and research with human subjects}
    \item[] Question: For crowdsourcing experiments and research with human subjects, does the paper include the full text of instructions given to participants and screenshots, if applicable, as well as details about compensation (if any)? 
    \item[] Answer: \answerNA{} % Replace by \answerYes{}, \answerNo{}, or \answerNA{}.
    \item[] Justification: The paper does not involve crowdsourcing or human subjects.

\item {\bf Institutional review board (IRB) approvals or equivalent for research with human subjects}
    \item[] Question: Does the paper describe potential risks incurred by study participants, whether such risks were disclosed to the subjects, and whether Institutional Review Board (IRB) approvals (or an equivalent approval/review based on the requirements of your country or institution) were obtained?
    \item[] Answer: \answerNA{} % Replace by \answerYes{}, \answerNo{}, or \answerNA{}.
    \item[] Justification: The paper does not involve human subjects or study participants.

\item {\bf Declaration of LLM usage}
    \item[] Question: Does the paper describe the usage of LLMs if it is an important, original, or non-standard component of the core methods in this research? Note that if the LLM is used only for writing, editing, or formatting purposes and does \emph{not} impact the core methodology, scientific rigor, or originality of the research, declaration is not required.
    %this research? 
    \item[] Answer: \answerNA{} % Replace by \answerYes{}, \answerNo{}, or \answerNA{}.
    \item[] Justification: The paper does not use LLMs as part of the core methodology.

\end{enumerate}

\end{document}